\documentclass[11pt,twoside]{article}
\usepackage{xspace,fullpage,amsmath,amsthm,amssymb,amsfonts,boxedminipage,microtype,hyperref}
\usepackage{paralist}
\usepackage{bm}
\usepackage{bbm}
\usepackage{array}
\usepackage{multirow}
\usepackage{times}
\usepackage{vfmacros}
\newtheorem{remark}[thm]{Remark}
\newtheorem{fact}[thm]{Fact}

\usepackage[style=alphabetic,backend=bibtex,maxbibnames=20,maxcitenames=6,firstinits=true,doi=false,url=false]{biblatex}
\newcommand*{\citet}[1]{\textcite{#1}}
\newcommand*{\citetall}[1]{\AtNextCite{\AtEachCitekey{\defcounter{maxnames}{999}}} \textcite{#1}}

\bibliography{vf-allrefs-central}

\newif\iffull
\fulltrue

\providecommand{\Zf}{{\mathcal Z}_f}

\newcommand{\sm}{\setminus}

\newcommand{\dc}{{\bar{\kappa}_2}}
\newcommand{\dcii}{{\bar{\kappa}_1}}
\newcommand{\dci}{{\kappa_1}}
\newcommand{\dcv}{{\kappa_v}}
\newcommand{\dcvi}{{\bar{\kappa}_v}}
\newcommand{\dcsp}{{\bar{\kappa}_2^2}}

\newcommand{\SD}{\mathrm{SD}}
\newcommand{\SQDIM}{\mathrm{SQDIM}}
\newcommand{\VCDIM}{\mathrm{VCdim}}
\newcommand{\RSD}{\mathrm{RSD}}
\newcommand{\RSDV}{\mathrm{RSD}}
\newcommand{\QC}{\mathrm{QC}}
\newcommand{\RQC}{\mathrm{RQC}}
\newcommand{\CRSD}{\mathrm{cRSD}}
\newcommand{\KL}{{\mathrm{KL}}}
\newcommand{\KLR}{R_{\mathrm{KL}}}
\newcommand{\STAT}{\mbox{STAT}}
\newcommand{\VSTAT}{\mbox{VSTAT}}
\newcommand{\vSTAT}{\mbox{vSTAT}}
\newcommand{\COMM}{\mbox{1-STAT}}
\newcommand{\conv}{\mbox{conv}}

\newcommand{\cO}{{\mathcal O}}
\newcommand{\cQ}{{\mathcal Q}}

\newcommand{\frc}{\mbox{\tt{-frac}}}

\newcommand{\icvr}{\dci\mbox{\tt{-cvr}}}
\newcommand{\ircvr}{\dci\mbox{\tt{-Rcvr}}}
\newcommand{\vrcvr}{\dcv\mbox{\tt{-Rcvr}}}

\newcommand{\dif}{\left| D[\phi] - D_0[\phi]  \right|}

\newcommand{\difp}[1]{\left| D[#1] - D_0[#1]  \right|}
\newcommand{\diff}[3]{\left| {#2}[#1] - {#3}[#1]  \right|}
\newcommand{\sdif}{\left|\sqrt{D[\phi]} - \sqrt{D_0[\phi]} \right|}

\newcommand{\Line}{\mathsf{Line}}
\newcommand{\err}{\mathtt{err}}
\newcommand{\Learn}{\mathcal{L}}
\providecommand{\GF}{{\mathbb Z}}

\title{A General Characterization of the Statistical Query Complexity}

\author{Vitaly Feldman \\
IBM Research - Almaden
}

\begin{document}

\date{}

\maketitle

\begin{abstract}
Statistical query (SQ) algorithms are algorithms that have access to an {\em SQ oracle} for the input distribution $D$ instead of i.i.d.~ samples from $D$. Given a query function $\phi:X \rar [-1,1]$, the oracle returns an estimate of $\E_{x\sim D}[\phi(x)]$ within some tolerance $\tau_\phi$ that roughly corresponds to the number of samples.

In this work we demonstrate that the complexity of solving an arbitrary statistical problem using SQ algorithms can be captured by a relatively simple notion of statistical dimension that we introduce. SQ algorithms capture a broad spectrum of algorithmic approaches used in theory and practice, most notably, convex optimization techniques. Hence our statistical dimension allows to investigate the power of a variety of algorithmic approaches by analyzing a single linear-algebraic parameter.
Such characterizations were investigated over the past 20 years in learning theory but prior characterizations are restricted to the much simpler setting of classification problems relative to a fixed distribution on the domain \cite{BlumFJ+:94,BshoutyFeldman:02,Yang:01,Yang:05,BalcazarCGKL:07,Simon:07,Feldman:12jcss,Szorenyi:09}. 
Our characterization is also the first to precisely characterize the necessary tolerance of queries. We give applications of our techniques to two open problems in learning theory and to algorithms that are subject to memory and communication constraints.
\end{abstract}

\thispagestyle{empty}
\newpage
\iffull
\tableofcontents
\thispagestyle{empty}
\newpage
\fi
\setcounter{page}{1}

\section{Introduction}
The statistical query model relies on an oracle that given any bounded function on a single domain element provides an estimate of the expectation of the function on a random sample from the input distribution $D$. Namely, for a query function $\phi:X\rar[-1,1]$ and {\em tolerance $\tau$}, the $\STAT_D(\tau)$ oracle responds with a value $v$ such that $|v-\E_{x\sim D}[\phi(x)]| \leq \tau$.

This model was introduced by \citet{Kearns:98} as a restriction of the PAC learning model \cite{Valiant:84}. \citet{Kearns:98} demonstrated that any learning algorithm that is based on statistical queries can be automatically converted to a learning algorithm robust to random classification noise. In addition, he showed that a number of known PAC learning algorithms can be expressed as algorithms using statistical queries instead of random examples themselves. Subsequently, many of algorithmic approaches used in machine learning theory and practice have been shown to be implementable using SQs (\eg \cite{BlumFKV:97,DunaganVempala:04,BlumDMN:05,ChuKLYBNO:06,FeldmanPV:13,BalcanF15}; see \cite{Feldman16:easq} for a brief overview) including most standard approaches to convex optimization \cite{FeldmanGV:15}. Indeed, solving linear equations over a finite field is the only known problem for which a superpolynomial separation between SQ complexity and the usual computational complexity is known \cite{Kearns:98} (or ever conjectured). Given random equations, this problem can be solved efficiently using Gaussian elimination (over a finite field), a technique that is too brittle for solving more realistic statistical problems\footnote{Here and below, by {\em statistical problem} we informally refer to any problem for which in the standard setting the input consists of i.i.d.~samples from some unknown input distribution $D$ (possibly from a restricted class of distributions $\D$) and the success criterion is defined relative to $D$ (and not the specific samples that were observed). A formal definition will be given later.}.

A special case of a statistical query is a {\em linear} (also referred to as {\em counting}) query on a dataset $S\in X^n$ which is defined in the same way as a statistical query relative to the uniform distribution on the elements of $S$. The problem of answering linear queries while preserving privacy of the individuals in the dataset played a fundamental role in the development of the notion of differential privacy \cite{DinurN03,BlumDMN:05,DworkMNS:06}. It remains a subject of intense theoretical and practical research in differential privacy since then (see \cite{DworkRoth:14} for a literature review and \cite{BlumDMN:05,GuptaHRU:11,FeldmanGV:15} for examples of application of SQ algorithms in this context). Further, access to an SQ oracle is known to be equivalent (up to polynomial factors) to local differential privacy model \cite{KasiviswanathanLNRS11} that has received much recent attention in industry \cite{ErlingssonPK14,Apple:16}. In the opposite direction: differentially private algorithms for answering linear queries were recently shown to imply algorithms for the challenging problem of answering adaptively chosen statistical queries \cite{DworkFHPRR14:arxiv,DworkFHPRR15:arxiv,BassilyNSSSU15}.

Other notable applications of SQ learning algorithms include derivation of theoretical and practical learning algorithm for distributed data systems \cite{ChuKLYBNO:06,RoySKSW10,Sujeeth:11,BalcanBFM12,SteinhardtVW16}. In this context it is known that access to an SQ oracle is equivalent (up to polynomial factors) to being able to extract only a limited number of bits from each data sample \cite{Ben-DavidD98,FeldmanGRVX:12,SteinhardtVW16}. This model is motivated by communication constraints in distributed systems and has been studied in several recent works \cite{ZhangDJW13,SteinhardtD15,SteinhardtVW16}.

\remove{
Since the introduction the model has been shown to be equivalent or closely-related to many other models and concepts (some discovered independently): linear statistical functionals \cite{Wasserman:2010}, learning with a distance oracle \cite{BenDavidIK:90}, approximate counting (or linear) queries studied extensively in differential privacy (\eg \cite{DinurN03,BlumDMN:05,DworkMNS:06,RothR10}), local differential privacy \cite{KasiviswanathanLNRS11}, evolvability \cite{Valiant:09,Feldman:08ev}, and algorithms that extract a small amount of information from each sample \cite{FeldmanGRVX:12,FeldmanPV:13,ZhangDJW13,GargMN14,SteinhardtD15,SteinhardtVW16}.  and practical applications (\eg \cite{ChuKLYBNO:06,RoySKSW10,Sujeeth:11,DworkFHPRR15:arxiv}).
}

A remarkable property of SQ algorithms is that it is possible (and in some cases relatively easy) to prove strong information-theoretic lower bounds on the complexity of any SQ algorithm that solves a given statistical problem. Given the considerable breadth and variety of approaches to statistical problems with provable guarantees that are known to be implementable using statistical queries (and only one known exception), this provides strong and unconditional evidence of the problem's hardness. In fact, for a number of central problems in learning theory and complexity unconditional lower bounds for SQ algorithms are known that closely match the known {\em computational} complexity upper bounds for those problems (\eg \cite{BlumFJ+:94, FeldmanGRVX:12,FeldmanPV:13,BreslerGS14a,DachmanFTWW:15,DiakonikolasKS:16}). SQ lower bounds are also known to directly imply strong structural lower bounds. For example, lower bounds against general convex relaxations of Boolean constraint satisfaction problems \cite{FeldmanPV:13,FeldmanGV:15}, lower bounds on approximation of Boolean functions by polynomials \cite{DachmanFTWW:15} and lower bounds on dimension complexity of Boolean function classes (which is closely related to sign-rank of matrices) \cite{Sherstov:08,FeldmanGV:15} are implied by SQ lower bounds.


\subsection{Prior work}
The SQ complexity of PAC learning was first investigated in a seminal work of \citetall{BlumFJ+:94}. They proved that the SQ complexity of weak PAC learning (that is, classification with a non-negligible advantage over the random guessing) of a function class $\C$ over a domain $X'$, relative to a fixed distribution $P$ on $X'$ is characterized (up to polynomials) by a simple linear-algebraic parameter called the {\em statistical query dimension} $\SQDIM(\C,P)$. Roughly, $\SQDIM(\C,D)$ measures the maximum number of ``nearly uncorrelated" (relative to $P$) functions in $\C$. Their characterization has been strengthened and simplified in several subsequent works \cite{Yang:01,BshoutyFeldman:02,BlumKW:03,Yang:05} and applied to a variety of problems in learning theory (\eg  \cite{BlumFJ+:94, KlivansSherstov:07a}).  Moreover the dimension itself was found to be tightly related to other notions of complexity of function classes and matrices such as margin complexity, sign-rank, approximate rank and discrepancy in communication complexity \cite{Simon:06,Sherstov:08, KlivansSherstov:10,KallweitSimon:11}.

Two obvious limitations of $\SQDIM$ are that it only characterized weak and fixed-distribution (also referred to as {\em distribution-specific}) SQ learning. The first limitation was addressed in \cite{BalcazarCGKL:07,Simon:07} who derived relatively involved characterizations of (strong) PAC learning. Subsequently, \citet{Feldman:12jcss} and \citet{Szorenyi:09} have found (different) relatively simple characterizations. The characterization in \cite{Feldman:12jcss} was also extended to a more general agnostic learning model \cite{KearnsSS:94} and has lead to better understanding of complexity of several learning problems \cite{FeldmanLS:11colt,GuptaHRU:11,DachmanFTWW:15}.

The second limitation is the fixing of the distribution $P$. It is much more challenging and as a result the SQ complexity of PAC learning is still poorly understood. A long-standing and natural open problem was to find a characterization of general (or distribution-independent) PAC learning (mentioned, for example, in \cite{KallweitSimon:11}). Associated with this problem is the question of whether the SQ complexity of learning $\C$ distribution-independently is equal to the maximum over all distributions $P$ of $\SQDIM(\C,P)$ \cite{KallweitSimon:11}. This is a natural conjecture since it holds for sample complexity of learning (there exists a distribution $P$ such that PAC learning relative to $P$ requires $\Omega(\VCDIM(\C))$ samples). It also holds for the hybrid SQ model in which the learner can get samples from $P$ (without the value of the target function) in addition to SQs \cite{FeldmanKanade:12}.

In a more recent work, \citetall{FeldmanGRVX:12} started a study of SQ algorithms outside of learning theory. They generalized the oracle of \citet{Kearns:98} (in a straightforward way) to any problem where the input is assumed to be random i.i.d.~samples from some unknown distribution.  They then described a notion of statistical dimension that generalized $\SQDIM$ to arbitrary statistical problems and showed that their dimension can be used to lower bound the SQ complexity of solving problems using SQ algorithms. Another important property of their dimension is that it treats the tolerance of queries separately from the query complexity. This was necessary to obtain a meaningful lower bound for the problem of recovering a planted bi-clique. In this problem the gap between the number of samples with which the problem becomes trivial and the number of samples for which the problem is believed to be computationally hard is just quadratic.

Further, to make the correspondence between the number of samples $n$ and the accuracy of queries precise,  \citet{FeldmanGRVX:12} introduced a strengthening of the SQ oracle that incorporates the variance of the random variable $\phi(x)$ into the estimate. More formally,  given as input any function $\phi : X \rightarrow [0,1]$, $\VSTAT_D(n)$  returns a value $v$ such that $|v -  p| \leq \max\left\{\frac{1}{n}, \sqrt{\frac{p(1-p)}{n}}\right\}$, where $p = \E_{x \sim D}[\phi(x)]$. Note that $\frac{p(1-p)}{n}$ is the variance of the empirical mean when $\phi$ is Boolean. More generally, the oracle can be used to estimate of the expectation $\E_{x \sim D}[\phi(x)]$ for any real-valued function $\phi$ within $\tilde O(\sigma/\sqrt{n})$, where $\sigma$ is the standard deviation of $\phi(x)$ \cite{Feldman:16sqvar}. The lower bounds in \cite{FeldmanGRVX:12} apply to this stronger oracle.  

While the dimension in \cite{FeldmanGRVX:12} allows proving lower bounds it does not capture the SQ complexity of a problem over distributions. Indeed, in a follow-up work \cite{FeldmanPV:13}, a stronger notion of dimension was necessary to get a tight lower bound for planted satisfiability problems. Their notion is based on so called discrimination norm and was also used to lower bound the SQ complexity of stochastic convex optimization \cite{FeldmanGV:15}. Still their dimension provides only a lower bound on the SQ complexity of problems.

\subsection{Overview of results}
We demonstrate that SQ complexity of an arbitrary statistical problem can be tightly captured using a linear-algebraic parameter that we refer to as (randomized) statistical dimension. In particular, we obtain nearly tight characterization for all many-vs-one decision problems, PAC learning and stochastic optimization. Unlike previously known characterizations, our characterization precisely captures the {\em estimation complexity}\footnote{The estimation complexity of SQ algorithm represents the number of samples necessary to give an answer to any single query of the oracle it uses. For algorithms with access to $\STAT_D(\tau)$ it is defined to be $1/\tau^2$ (since $O(1/\tau^2)$ samples suffice to get such an estimate with high probability); for algorithms with access to $\VSTAT_D(n)$ it is defined to be $n$.} of SQ algorithms with both $\STAT$ and $\VSTAT$ oracles. Previous approaches characterized only the maximum of query and estimation complexity.


The existence of such parameter for general statistical problems is rather surprising since SQ algorithms can query the oracle adaptively (that is, every query can depend arbitrarily on responses to previous queries) and many SQ algorithms require such adaptivity. Measuring query complexity in models that allow adaptive queries usually requires dealing with arbitrarily deep sequences of alternating $\exists$ and $\forall$ quantifiers that are rarely amenable to accurate analysis. Indeed, we do not know if there exists a parameter that captures the SQ complexity {\em precisely} while avoiding such quantification. Our results demonstrate that SQ complexity of an arbitrary statistical problem is approximated well by a much simpler notion.

For several types of statistical problems, existing characterizations of statistical query complexity have been used to reveal important structural properties that accurately correspond to the known bounds on {\em computational complexity} of these problems. For example, the number of approximately uncorrelated functions for distribution-specific PAC learning \cite{BlumFJ+:94}, approximate resilience for agnostic learning relative to a product distribution \cite{DachmanFTWW:15} and the degree of independence of a distribution over predicates for planted constraint satisfaction problems \cite{FeldmanPV:13}. 
Our new characterization suggests that such structural properties are likely to exists for many other types of statistical problems.  Finding these properties for computationally hard statistical problems is an interesting avenue for further research that might shed light on the complexity of many important theoretical and practical problems. Towards this goal, a considerable part of this work is devoted to deriving simplifications of our characterization for more specific types of problems (such as optimization and learning) and to variants of the dimension that might be easier to analyze when less precise characterization is sufficient. We also relate our notion of statistical dimension to known techniques for proving lower bounds on SQ complexity.

Our characterization also implies the existence of a SQ algorithm with specific universal structure for every problem that can be solved using SQs (albeit not a computationally efficient one). The existence of such algorithms can be used to derive new properties of SQ algorithms. One example of such application is algorithms for the memory-limited streaming that we describe in Appendix \ref{sec:memory}. Another application is a reduction from $k$-wise queries to regular queries that appears in a subsequent work \textcite{FeldmanGhazi:17}. In both cases our universal algorithm gives an exponential improvement over prior results for these problems.



\paragraph{Decision problems:} We start with the relatively simple case of many-vs-one decision problems (Section \ref{sec:stat-decision}). These are problems specified by a set of distributions $\D$ over a domain $X$ and a {\em reference} distribution $D_0$ over $X$. Given access to an input distribution $D \in \D \cup \{D_0\}$ the goal is to decide whether $D\in \D$ or $D=D_0$ (in the standard setting the access is to i.i.d.~samples from $D$ whereas in our case the access will be via a SQ oracle). We denote this problem by $\B(\D,D_0)$.
\iffull
For example, we can take $\D$ to be the set of all distributions over $k$-SAT clauses whose support can be satisfied by some assignment and $D_0$ to be the uniform distribution over $k$-clauses. Then  $\B(\D,D_0)$ is exactly the stochastic version of the $k$-SAT refutation problem.
\fi

An important property of decision problems is that their deterministic SQ complexity has a simple and sharp characterization in terms of the number of functions that can distinguish between $D_0$ and any distribution in $\D$. Specifically, let $d$ be the smallest integer $d$ such that there exist $d$ functions $\phi_1,\ldots,\phi_d:X\rar \pmr$, such that for every $D \in \D$ there exists $i\in[d]$ satisfying $\left| D[\phi_i]-D_0[\phi_i]\right| > \tau$ (where $D[\phi_i] \doteq \E_{x\sim D}[\phi_i(x)]$). We refer to such set of functions as a {\em $\tau$-cover} of $\D$ relative to $D_0$. It is not hard to see that a decision problem $\B(\D,D_0)$ can be solved using $d$ queries to $\STAT_D(\tau)$ if and only if it has a $\tau$-cover of size $d$ (here and below ignoring multiplicative constants).

Unfortunately, proving lower bounds directly on the size of a $\tau$-cover is relatively hard due to a quantifier over $d$ functions. The first of the key ideas in our characterization is to consider a relaxation referred to as
a randomized $\tau$-cover. A randomized $\tau$-cover of size $d$ is a distribution $\cP$ over functions from $X$ to $[-1,1]$ with the property that for every $D \in \D$,
$$\pr_{\phi \sim \cP}\lb \dif > \tau \rb \geq \fr{d}.$$
It is a relaxation of the (deterministic) $\tau$-cover that is equivalent to a classical notion of fractional cover.

We show that the size of the smallest randomized cover exactly characterizes the complexity of solving $\B(\D,D_0)$ by a randomized SQ algorithm. While the size of the smallest randomized $\tau$-cover appears even harder to analyze than the size of a $\tau$-cover, we simplify it using the dual notion. Formally, for a measure $\mu$ over the set $\D$ we define the maximum $\tau$-covered $\mu$-fraction as \equn{
\dci\frc(\mu,D_0,\tau) \doteq \max_{\phi:X\rar[-1,1]} \left\{\pr_{D\sim \mu}[ \dif > \tau] \right\}.
} and the corresponding {\em randomized statistical dimension} with $\dci$-discrimination as
   $$\RSD_\dci(\B(\D,D_0),\tau) \doteq \sup_{\mu \in S^\D} (\dci\frc(\mu,D_0,\tau))^{-1}.$$
The duality between the randomized $\tau$-covers and $\RSD_\dci(\B(\D,D_0),\tau)$ (given in Lem.~\ref{lem:stat-SD-is-rcover}) implies that we have obtained a characterization of the SQ complexity of decision problems without any significant overheads (and without having to explicitly deal with covers). Formally, we denote the smallest number of queries required to solve a problem $\Z$ using oracle $\cO$ with success probability $\beta$ by $\RQC(\Z,\cO,\beta)$. Our characterization states that (Thm.~\ref{thm:random-algorithm2queries}):
 $$\RQC(\B(\D,D_0),\STAT(\tau),1-\delta) \geq \RSD_\dci(\B(\D,D_0),\tau) \cdot (1-2\delta)\mbox{ and}$$
$$\RQC(\B(\D,D_0),\STAT(\tau/2),1-\delta) \leq \RSD_\dci(\B(\D,D_0),\tau) \cdot \ln(1/\delta).$$

The upper bound can be made deterministic by setting $\delta < 1/|\D|$. This implies that $\RSD_\dci(\B(\D,D_0),\tau)$ characterizes the deterministic SQ complexity of solving $\B(\D,D_0)$ with access to $\STAT(\tau)$ up to a $\ln(|\D|)$ factor in query complexity. For some problems, such as the decision versions of the planted bi-clique and planted satisfiability problems studied in \cite{FeldmanGRVX:12,FeldmanPV:13} this characterization is reasonably tight. At the same time for some of the most common and interesting problems, $\ln(|\D|)$ is too large. For example in (distribution-independent) PAC learning we need to deal with $\D$ which includes all distributions\footnote{The set of all distribution is, of course, infinite but can be replaced with a suitable $\eps$-net. The net will have size $\eps^{-|X|}$ for some small $\eps$. If $X$ itself is infinite one also first needs to define an $\eps$-net on $X$.} over some large domain $X$. In this case $\ln(|\D|)=\Omega(|X|)$ making the characterization meaningless.

\paragraph{General (search) problems:}
To extend our statistical dimension to more general statistical problems we start by defining them formally. We define a search problem over distributions by a set of input distributions $\D$,  a set of solutions $\F$ and a function $\Z:\D \rightarrow 2^{\F}$. For $D \in \D$, $\Z(D) \subseteq \F$ is the (non-empty) set of valid solutions for $D$. The goal of an algorithm is to find a valid solution $f \in \Z(D)$ given access to an (unknown) input distribution $D\in \D$. For a solution $f \in \F$, we let $\Zf \doteq \{D\in\D \cond f \in \Z(D)\}$ be the set of distributions in $\D$ for which $f$ is a valid solution. Note that this general formulation captures most formal models used in machine learning and statistics for problems over datasets consisting of i.i.d.~samples (see Appendix \ref{sec:problem-examples} for some specific examples).

We characterize the statistical dimension of such search problems using the statistical dimension of the hardest many-to-one decision problem implicit in the search problem. This is a common approach for proving lower bounds in general and was also used in previous lower bounds for SQ algorithms (\eg \cite{Feldman:12jcss,FeldmanGRVX:12}). We show that, remarkably, the converse also holds for SQ algorithms: the hardest many-to-one decision problem is essentially as hard as the search problem. The key idea is
that an algorithm for solving a problem $\Z$ should, for every $D_0$ and $D\in \D$, either output a valid solution for $D$ given access to $D_0$ instead (of $D$) or generate a query that distinguishes between $D_0$ and $D$. If the former condition is true, then we can solve the problem using $D_0$. Otherwise, we can use the distinguishing query to make progress toward reconstructing the input distribution $D$. The reconstruction is done using the classic Multiplicative Weights algorithm which allows to reconstruct the input distribution by solving at most $O(\KLR(\D)/\tau^2)$ decision problems, where $\KLR(\D)$ measures the ``radius" of $\D$ in terms of KL-divergence. This radius is at most $\ln(|X|)$ but is much smaller for many problems. Our approach is inspired by the use of a simpler (distribution-specific) reconstruction algorithm in \cite{Feldman:12jcss} and the use of Multiplicative Weights algorithm to answer statistical and counting queries \cite{HardtR10,DworkFHPRR14:arxiv} (although there is no direct connection between that problem and ours).

This technique is sufficient to get a characterization of the deterministic SQ complexity of search problems. All one needs is to define
$$\SD_\dci(\Z,\tau) \doteq \sup_{D_0 \in S^X} \inf_{f\in \F} \RSD_\dci(\B(\D\setminus \Zf,D_0),\tau),$$
where $S^X$ denotes the set of all distributions over $X$. In Theorems~\ref{thm:stat-search-lower} and \ref{thm:stat-search-upper} we prove that $\SD_\dci(\Z,\tau)$ characterizes the query complexity of solving $\Z$ with access to $\STAT_D(\tau)$ up to a factor of $O(\log |\D| \cdot \KLR(\D)/\tau^2)$.

To avoid the problematic $\log(|\D|)$ factor in the upper bound and to ensure that the lower bound holds against randomized algorithms, we need to deal with the substantially more delicate randomized case.  We show that it is possible to give a nearly tight characterization by considering ``fractional" solutions. More formally, for a probability measure $\cP$ over $\F$ and $\alpha > 0$, we define the set of distributions for which $\cP$ provides a solution with probability at least $\alpha$ by $\Z_\cP(\alpha) \doteq \left\{D \in \D \cond \cP(\Z(D)) \geq \alpha \right\}$. We then define the randomized statistical dimension for success probability $\alpha$ as the complexity of the hardest decision problem, where we first eliminate all the input distributions for which there exists a randomized algorithm with success probability $\geq \alpha$ that does not look at the input distribution:
  $$\RSD_\dci(\Z,\tau,\alpha) \doteq \sup_{D_0 \in S^X}\inf_{\cP \in S^\F} \RSD_\dci(\B(\D\setminus \Z_\cP(\alpha),D_0),\tau) .$$
When this dimension is equal to $d$ we prove that for every $1\geq  \beta>\alpha>0$, $\delta>0$ :
$$\RQC(\Z,\STAT(\tau),\beta) \geq d \cdot (\beta - \alpha) \mbox { and }$$
 $$\RQC(\Z,\STAT(\tau/3),\alpha-\delta) = \tilde O\lp  d \cdot \frac{\KLR(\D)}{\tau^2} \cdot\log(1/\delta)\rp .$$

We remark that a different approach for dealing with solutions in the randomized case is used in the lower bound technique of \citet{FeldmanGRVX:12}. Their approach does not appear to suffice for an upper bound.

While this dimension is somewhat cumbersome, it can be substantially simplified for problems where one can verify the solution using a statistical query (such as in PAC learning or planted constraint satisfaction problems) or estimate the value of the solution in an optimization setting. In this case, the term $\Z_\cP(\alpha)$ can be removed by maximizing only over reference distributions that cannot pass the verification step (see Sec.~\ref{sec:search-special} for more details). We define our statistical dimension for PAC learning on the basis of this simplification.

\paragraph{$\VSTAT$:}
Dealing directly with the accuracy guarantees of $\VSTAT_D(n)$ in the type of results that we give for $\STAT_D(\tau)$ would be rather painful both due to a more involved expression for accuracy and the fact that the expression is asymmetric: a query function that distinguishes $D_0$ from $D$ might not distinguish $D$ from $D_0$ since the tolerance depends on the expectation with respect to the input distribution.
We show that the analysis of $\VSTAT$ can be greatly simplified by introducing a symmetric oracle that we show to be equivalent (up to a factor of 3) to $\VSTAT$ (Lem.~\ref{lem:vstat-reduction}).  The statistical query oracle $\vSTAT_D(\tau)$ is an oracle that given a function $\phi : X \rightarrow [0,1]$ returns a value $v$ such that $\left|\sqrt{v} - \sqrt{D[\phi]}\right| \leq \tau$.
Now, by defining the maximum $\tau$-covered $\mu$-fraction as \equn{
\dcv\frc(\mu,D_0,\tau) \doteq \max_{\phi:X\rar[-1,1]} \left\{\pr_{D\sim \mu}\lb\left|\sqrt{D_0[\phi]} - \sqrt{D[\phi]}\right| > \tau\rb \right\}}
and using it in place of $\dci\frc(\mu,D_0,\tau)$ to define randomized statistical dimension we can obtain analogous characterizations for the complexity of solving decision and search problems using $\vSTAT_D(\tau)$ (see Thms.~\ref{thm:vstat-search-lower-random} and \ref{thm:vstat-search-upper-random}). We refer to these dimensions with subscript $\kappa_v$ instead of $\kappa_1$. This ``trick" also gives a new perspective on $\VSTAT_D(n)$  (and consequently on the length of the standard confidence interval for the bias of a Bernoulli r.v.) as an oracle that ensures, up to a constant factor, fixed tolerance for the estimation of standard deviation of the corresponding Bernoulli r.v. (that is, $\phi(x)$ when $\phi$ is Boolean).

\paragraph{Average discrimination and relationship to known bounds:}
The dimensions that we have defined can often be analyzed relatively easily. In a number of problems, for an appropriate choice of $D_0$ and $\mu$ (which is usually just uniform over some subset of $\D$) we get that $\dif$ (or $\sdif$) is strongly concentrated around some value $\tau_0$ when $D$ is chosen according to $\mu$. This implies that the maximum $\tau$-covered fraction can be upper-bounded directly by the statement of concentration. 
However in some cases it is still analytically more convenient to upper bound the average value by which a query distinguishes between distributions instead of the fixed minimum $\tau$. That is, instead of
 $$\max_{\phi:X\rar[0,1]} \left\{\pr_{D\sim \mu}\lb \sdif > \tau\rb \right\},$$ it is often easier to analyze the largest covered fraction of distributions that have a larger than $\tau$ average discrimination $$\dcvi(\mu,D_0) \doteq
 \max_{\phi: X\rar[0,1]} \left\{\E_{D\sim \mu}\left[ \sdif \right]\right\}$$ to which we refer as {\em $\dcvi$-discrimination}.
For $\D' \subseteq \D$ let $\mu_{|\D'} \doteq \mu(\cdot \cond \D')$. The maximum covered $\mu$-fraction and the randomized statistical dimension for $\dcvi$-discrimination are defined as
\equn{
\dcvi\frc(\mu,D_0,\tau) \doteq \max_{\D' \subseteq \D} \left\{\left. \mu(\D')\ \right|\  \dcvi(\mu_{|\D'},D_0)> \tau \right\}.}
$$\RSD_\dcvi(\B(\D,D_0),\tau) \doteq \sup_{\mu \in S^\D} \lp\dcvi\frc(\mu,D_0,\tau)\rp^{-1}.$$
An analogous modification can be made to the statistical dimension of search problems.
This additional relaxation is (implicitly) used in the statistical dimension in \cite{FeldmanPV:13} and in most earlier works on SQ dimensions. We show that this average version behaves in almost the same way as the strict version with the only difference being that the upper-bounds grow by a factor of $1/\tau$. The implication of this is that, whenever it is more convenient, the average version of discrimination can be used without significant loss in the tightness of the dimension. See Sec.~\ref{sec:statkv} for additional details.

The dimension defined in this way can be easily lower-bounded by a number of notions that were studied before, including the discrimination norm in \cite{FeldmanPV:13}, average correlation from \cite{FeldmanGRVX:12} (which itself can be upper-bounded by pairwise-correlations based notions) and weighted spectral norm that was used in \cite{Yang:05}. This provides examples of the analysis of our dimensions and gives a unifying view on several of the prior techniques. See Sec.~\ref{sec:norms} for additional details.

\paragraph{Combined dimension:}
In some cases one is interested in a coarser picture in which it is sufficient to estimate the maximum of the query complexity and the estimation complexity up to a polynomial. In fact known analyses of SQ complexity in the context of distribution-specific PAC learning give bounds only on this combined notion of complexity. For such cases we can avoid our fractional notions and get a simpler {\em combined statistical dimension} based on average discrimination. We base the notion on the average version of $\dci$, denoted by $\dcii$. For decision problems we get the following simplified dimension:
$$\CRSD_\dcii(\B(\D,D_0)) \doteq \sup_{\mu \in S^\D} \lp \dcii(\mu,D_0)\rp^{-1}.$$
To show that the combined dimension characterizes SQ complexity (up to a polynomial) we demonstrate that it can be related to $\RSD_\dci$. We also extend the combined dimension to search problems. See Sec.~\ref{sec:combined} for additional details.

\subsection{Applications}
To illustrate some of the concepts that we introduced, we describe several applications of both our upper and lower bounds. Additional applications can be found in \cite{FeldmanGhazi:17}.
\paragraph{Separation of distribution-specific and distribution independent SQ learning:}
We describe a simplification of our characterizations for {\em distribution-independent} PAC learning problems (Sec.~\ref{sec:learning-dim}). We then use the simplified version of $\CRSD$ to prove the first lower bound on the SQ complexity PAC learning that holds only in the distribution-independent setting. Specifically, we consider the class of functions that are lines on a finite field plane: for $a,z\in \GF_p^2$, $\ell_{a}(z) = 1$ if and only if $a_1z_1 + a_2 = z_2 \mod p$. Then $\Line_p \doteq \{\ell_a \cond a \in \GF_p^2\}$. We prove that any SQ algorithm for (distribution-independent) PAC learning of $\Line_p$ with error $\eps = 1/2-c \cdot p^{-1/4}$ (for some constant $c$), has SQ complexity of $\Omega(p^{1/4})$ (Thm.~\ref{thm:line-lower-bound}). Our analysis of the resulting dimension uses the average correlation technique from \cite{FeldmanGRVX:12}.

This lower bound allows us to resolve the question about the relationship between SQ complexity of distribution independent PAC learning and the maximum over all distributions of the complexity of distribution-specific learning. We show that the former cannot be upper bounded by any function of the latter. To prove the upper bound, we describe a fairly simple distribution-specific learning algorithm for $\Line_p$ that has SQ complexity of $O(1)$ for any constant $\eps >0$ (Thm.~\ref{thm:line-upper}). At a high level, knowing the distribution allows the learner to identify a small number of candidate hypotheses, one of which is guaranteed to be close to the unknown function. The maximum over all distributions of the SQ complexity of distribution-specific learning is also known to be equal to the complexity of PAC learning in the hybrid SQ model in which the algorithm can observe unlabeled samples in addition to making queries \cite{FeldmanKanade:12} and hence our result also separates between the hybrid and the usual\footnote{While the hybrid model was discussed in some early work on the SQ model and used in the first algorithm for SQ learning of halfspaces \cite{BlumFKV:97}, it ended up not being necessary for solving that problem \cite{DunaganV08} or in any other learning algorithms.} SQ models.

We remark that our separation also implies the strong separation of sign-rank (also referred to as dimension complexity) and VC dimension recently proved by \citetall{AlonMY16} using the same $\Line_p$ class of functions. By the results in \cite{DunaganV08}, SQ complexity lower bounds (up to polynomials) the sign-rank while the VC-dimension is a lower bound on distribution-specific SQ complexity of learning for some distribution \cite{BlumFJ+:94}. A weaker (exponential) separation of sign-rank and VC dimension using the class of parity functions was first obtained in Forster's breakthrough result \cite{Forster:02} (and, as pointed out in \cite{FeldmanGV:15}, is also implied by known results on SQ complexity of learning parities and halfspaces \cite{BlumFJ+:94,BlumFKV:97}).

\paragraph{Separation of noise tolerant learning and SQ learning:}
Our lower bound gives a second example of a class of functions  that is easy to learn using random examples but hard for statistical queries (the first being the parity functions\footnote{However, these classes are closely related since they are special cases of linear subspaces over a finite field and learning algorithms rely on Gaussian elimination.} \cite{Kearns:98}). The separation is stronger than that for parities since the VC-dimension of $\Line_p$ is just 2 and a constant number of samples suffices for learning (for any constant error).

Further, it is easy to see that $\Line_p$ can be learned efficiently with random classification noise. Thus our lower bound makes progress on the long standing open problem of \citet{Kearns:98} who asked whether efficient learning with random classification noise can be separated from efficient SQ learning. This question was addressed in an influential work of \citetall{BlumKW:03} who used a subclass of parity functions to give a separation when the noise rate is relatively low. More formally, their algorithm has a super-polynomial dependence on $1/(1-2\eta)$ where $\eta$ is the noise rate. Efficient learning with random classification noise requires polynomial dependence on this parameter \cite{AngluinLaird:88} and any efficient SQ algorithm gives an efficient noise-tolerant learning algorithm \cite{Kearns:98}. In addition, our lower bound is exponential in the input size as opposed to $n^{\Omega(\log\log n)}$ lower bound in \cite{BlumKW:03}.

We note however that the separation in \cite{BlumKW:03} is for distribution-specific SQ learning. Therefore it remains open whether the stronger separation of SQ learning from noise tolerant learning can be obtained in this more restrictive setting. See Sec.~\ref{sec:line-noise} for more details.

\paragraph{Applications to other models:}
Our results can be easily translated into a number of related models. For example, we obtain a characterization, up to a polynomial, of the sample complexity of solving a problem over distributions with limited communication from every sample (such as in distributed data access or in a sensor network).

Formally, for integer $b >0$, in this model we have access to $\COMM(b)$ oracle
\footnote{This oracle is also referred to as $b\mbox{-wRFA}$ in \cite{Ben-DavidD98} and $\mbox{1-MSTAT}(2^b)$ in \cite{FeldmanPV:13}.} for a distribution $D$ that given any function $\phi: X \rar \zo^b$,  takes an independent random sample $x$ from $D$ and returns $h(x)$. Learning with this oracle and related models have been studied in a number of recent works \cite{FeldmanGRVX:12,ZhangDJW13,FeldmanPV:13,SteinhardtD15,SteinhardtVW16}. This model is known to be equivalent to the randomized SQ model up to a polynomial and $2^b$ factors \cite{Ben-DavidD98,FeldmanGRVX:12,FeldmanPV:13,SteinhardtVW16}. Therefore our characterization immediately implies a characterization for this model. For completeness we include the details in Appendix \ref{sec:comm}. This result relies on our combined randomized SQ complexity for the $\VSTAT$ oracle. An analogous characterization also holds for the local differential privacy model that is known to be polynomially equivalent to the SQ model \cite{KasiviswanathanLNRS11}.

\citetall{SteinhardtVW16} showed that upper bounds on SQ complexity of solving a problem imply upper bounds on the amount of memory needed in the streaming setting. In this setting at step $i$ an algorithm observes sample $x_i$ drawn i.i.d.~from the input distribution $D$ and updates its state from $S_i$ to $S_{i+1}$, where for every $i$, $S_i \in \zo^b$. They show that any algorithm using $q$ queries to $\STAT(\tau)$ can be implemented using $O(\log (q/\tau) \cdot \log(|\D|))$ bits of memory and apply it to obtain an algorithm for sparse linear regression in this setting. The factor of $\log(|\D|)$ in their upper bound substantially limits the range of regression problems that can be addressed.

Implicit in the proof of our characterization is a way to convert any SQ algorithm for a problem $\Z$ into a SQ algorithm for $\Z$ with a specific simple structure . It turns out that it is easy to implement algorithms with such structure in the memory-limited streaming setting. Our implementation requires  $O(\log q \cdot \KLR(\D)/\tau^2)$ bits of memory, which is an exponential improvement over the $\log(|\D|)$ dependence in many settings of interest. Consequently, we can substantially extend the range of sparse linear regression problems which can be solved in this settings. Additional details on this applications are in Appendix \ref{sec:memory}.

\section{Preliminaries}
\label{sec:prelims}
For integer $n\geq 1$ let $[n]\doteq \{1,\ldots, n\}$.
For a distribution $D$ over a domain $X$ and a function $\phi:X\rar \R$ we use $D[\phi]$ to refer to $\E_{x \sim D}[\phi(x)]$.
We denote the set of all probability distributions over a set $X$ by $S^X$.

\subsection{Problems over distributions}
We first define several general classes of problems.
For a set of distributions $\D$ over a domain $X$ and a reference distribution $D_0 \not\in \D$ over $X$, the {\em distributional decision problem} $\B(\D,D_0)$ is to decide given access (samples from or an oracles for) an unknown input distribution $D \in \D \cup \{D_0\}$, whether $D \in \D$ or $D = D_0$.

Let $\D$ be a set of distributions over $X$ let $\F$ be a set of solutions and $\Z:\D \rightarrow 2^{\F}$ be a map from a distribution $D \in \D$ to a non-empty subset of solutions $\Z(D) \subseteq \F$ that are defined to be valid solutions for $D$.
In a {\em distributional search problem} $\Z$ over $\D$ and $\F$ the goal is to find a valid solution $f \in \Z(D)$ given access to random samples or an oracle access to an unknown $D \in \D$. For a solution $f \in \F$, we denote by $\Zf$ the set of distributions in $\D$ for which $f$ is a valid solution.

Next we describe two important special cases of distributional search problems.
For $\eps > 0$, a linear optimizing search problem $\Z$ is a search problem over $\F$ and $\D$ such that every $f \in \F$ is associated with a function $\phi_f:X\rar [0,1]$ and for every $D\in \D$ and parameter $\eps >0$, $$\Z_\eps(D) \doteq \left\{h\ \left|\ D[\phi_h] \leq \min_{f \in \F} D[\phi_f] +\eps\right.\right\}.$$
(Other notions of approximation can also be considered but for brevity and simplicity we focus on additive approximation.)

Next we define problems where it is easy to verify the solution. We say that a search problem $\V$ is verifiable if for every $f\in \F$ there is an associated query function $\phi_f: X\rar [0,1]$ such that $\V$ with parameter $\theta$ is defined as
$$\V_\theta(D) \doteq \left\{f\ \left|\ D[\phi_f] \leq \theta \right.\right\}.$$

We note that the definition of verifiable and optimizing search can be generalized to the setting where instead of $D[\phi_f]$ we use the output of some (relatively-simple) SQ algorithm on the input distribution $D$. With minor modifications, the results in this work easily extend to this more general setting.

Some examples of problems over distributions that have been explored in the context of SQ model are included in Appendix \ref{sec:problem-examples}.
\subsection{Statistical queries}
The algorithms we consider here have access to a statistical query oracle for the input distribution. The most commonly studied SQ oracle was introduced by \citet{Kearns:98} and gives an estimate of the mean of any bounded function with fixed tolerance.

\begin{defn}
 Let $D$ be a distribution over a domain $X$, $\tau >0$ and $n$ be an integer. A statistical query oracle $\STAT_D(\tau)$ is an oracle that given as input any function $\phi : X \rightarrow [-1,1]$, returns some value $v$ such that
 $|v -  \E_{x \sim D}[\phi(x)]| \leq \tau$.
\end{defn}
We will also study a stronger oracle that captures estimation of the mean of a random variable from samples more accurately and was introduced in \cite{FeldmanGRVX:12}.

\begin{defn}
A statistical query oracle $\VSTAT_D(n)$ is an oracle that given as input any function $\phi : X \rightarrow [0,1]$ returns a value $v$ such that $|v -  p| \leq \max\left\{\frac{1}{n}, \sqrt{\frac{p(1-p)}{n}}\right\}$, where $p \doteq D[\phi]$.
\end{defn}

\remove{
\begin{remark}
For any function $\phi : X \rightarrow [-1,1]$ using two queries to $\VSTAT(2n)$ one can obtain a value $v$ such that $|v -  \E_{x\sim D}[\phi(x)]| \leq \max\left\{\frac{1}{n}, \sqrt{\frac{\E_{x\sim D}[|\phi(x)|]}{n}}\right\}$.
\end{remark}
}

One way to think about $\VSTAT$ is as providing a confidence interval for $p$, namely $[v - \tau_v,v +\tau_v]$, where $\tau_v \approx \max\{1/n,\sqrt{(v(1-v)/n}\}$. The accuracy $\tau_v$ that $\VSTAT$ ensures corresponds (up to a small constant factor) to the width of the standard confidence interval (say, with 95\% coverage) for the bias $p$ of a Bernoulli random variable given $n$ independent samples (\eg Clopper-Pearson interval \cite{ClopperP1934}). Therefore, at least for Boolean queries, it captures precisely the accuracy that can be achieved when estimating the mean using random samples. In contrast, $\STAT$ captures the accuracy correctly only when $p$ is bounded away from $0$ and $1$ by a positive constant.

\begin{remark}
\label{rem:vstat-simplify}
For convenience, in this work we will rely on a slightly weaker definition of $\VSTAT$ that returns a value $v$ such that $|v -  p| \leq \max\left\{\frac{1}{n}, \sqrt{\frac{p}{n}}\right\}$. Note that if $p\leq 1/2$ then our version with parameter $2n$ will be at least as accurate as the original one. When $p > 1/2$, we can use the query $1-\phi$ to the weaker version and return one minus its response. This ensures the same accuracy in this case. If we do not know a priori whether $p\leq 1/2$ we can ask both queries. The responses are also sufficient for picking which of the responses to use.
\end{remark}

We say that an algorithm is {\em statistical query} (SQ) if it does not have direct access
to $n$ samples from the input distribution $D$, but instead makes calls to a statistical query oracle for the input distribution. In this case we simply say that the algorithm has access to $\VSTAT(n)$ or $\STAT(\tau)$ (omitting the input distribution from the subscript).

Clearly $\VSTAT_D(n)$ is at least as strong as $\STAT_D(1/\sqrt{n})$ (but no stronger than $\STAT_D(1/n)$).
The {\em estimation complexity} of a statistical query algorithm using $\VSTAT_D(n)$ is the value $n$ and for an algorithm using $\STAT(\tau)$ it is $n=1/\tau^2$.
The query complexity of a statistical algorithm is the number of queries it uses. The {\em SQ complexity} of solving a problem $\Z$ with some SQ oracle $\cO$ is the lowest query complexity that can be achieved by an algorithm that solves the problem given access to $\cO$. We denote it by $\QC(\Z,\cO)$. For randomized algorithms the complexity naturally depends on the success probability and we denote the complexity of solving $\Z$ with success probability $\beta$ by $\RQC(\Z,\cO,\beta)$.

\remove{

Note that the estimation complexity corresponds to the number of i.i.d.~samples sufficient to simulate the oracle for a single query with at least some positive constant probability of success. However it is not necessarily true that the whole algorithm can be simulated using $O(n)$ samples since answers to many queries need to be estimated. Answering $m$ fixed (or non-adaptive) statistical queries can be done using $O(\log m \cdot n)$ samples but when queries depend on previous answers the best known bounds require $O(\sqrt{m} \cdot n)$ samples (see \cite{DworkFHPRR14:arxiv} for a detailed discussion). This also implies that a lower bound on sample complexity of solving a problem does not directly imply lower bounds on estimation complexity of a SQ algorithm for the problem.
}

\section{Decision problems}
\label{sec:stat-decision}
We first focus on the simpler case of many-vs-one decision problems.
\subsection{Deterministic dimension for decision problems}
\label{sec:stat-decision-det}
As a brief warm-up we start with a simple but weaker characterization of the deterministic complexity of solving decision problems.
The key property of decision problems is that their deterministic SQ complexity has a simple and sharp characterization in terms of the size of a certain cover by distinguishing functions. Specifically, we define:
\begin{defn}
For a set of distributions $\D$ and a reference distribution $D_0$ over $X$, $\icvr(\D,D_0,\tau)$ is defined to be the smallest integer $d$ such that there exist $d$ functions $\phi_1,\ldots,\phi_d:X\rar \pmr$, such that for every $D \in \D$ there exists $i\in[d]$ satisfying $\left| D[\phi_i]-D_0[\phi_i]\right| > \tau$.
\end{defn}
The following lemma was proved in \cite{Feldman:12jcss} in the context of PAC learning.
\begin{lem}
\label{lem:det-algorithm2queries}
Let $\B(\D,D_0)$ be a decision problem and $\tau >0$. Then $$\QC(\B(\D,D_0),\STAT(\tau)) \geq \icvr(\D,D_0,\tau) \geq \QC(\B(\D,D_0),\STAT(\tau/2)) .$$
\end{lem}
\begin{proof}
Let $\A$ be the algorithm that solves $\B(\D,D_0)$ using $q$ queries to $\STAT(\tau)$.  We simulate $\A$ by answering any query $\phi:X \rightarrow [-1,1]$ of $\A$ with value $D_0[\phi]$. Let $\phi_1,\phi_2,\ldots,\phi_q$ be the queries asked by $\A$ in this (non-adaptive) simulation. By the correctness of $\A$, the output of $\A$ in this simulation must be ``$D=D_0$". Now let $D$ be any distribution in $\D$. If we assume that for every $i \in [q]$, $|D[\phi_i] - D_0[\phi_i]| \leq \tau$, then the responses in our simulation are valid responses of $\STAT_D(\tau)$. Namely, for all $i$ the response of our simulated oracle is a value that is within $\tau$ of $D[\phi_i]$. By the correctness of $\A$, the simulation must then output ``$D\in \D$". The contradiction implies that $\icvr(\D,D_0,\tau) \leq q$.

For the other direction, let $\phi_1,\ldots,\phi_q: X \rar [-1,1]$ be the set of functions such that for every distribution $D' \in \D$ there exists $i \in [q]$ for which $|D'[\phi_i] - D_0[\phi_i]| > \tau$. For every $i\in [q]$ we ask the query $\phi_i$ to $\STAT(\tau/2)$ and let $v_i$ be the response. If exists $i$ such that $|v_i - D_0[\phi_i]| > \tau/2$ then we conclude that the input distribution is not $D_0$. Otherwise we output that the input distribution is $D_0$. By the definition of $\STAT$ this algorithm will be correct when $D=D_0$. Further, if $D\in \D$, then for some $i$, $|D[\phi_i] - D_0[\phi_i]| > \tau$, which implies that
\equn{|v_i - D_0[\phi_i]| \geq |D[\phi_i] - D_0[\phi_i]| - |v_i - D[\phi_i]| > \tau/2.}
This ensures that for all distributions in $\D$ the output of the algorithm will be correct.
\end{proof}

Unfortunately, proving lower bounds directly on the size of a $\icvr$ appears to be hard.
A simple way around it is to analyze (the inverse of) the largest covered fraction of distributions, that is $$\lp\max_{\D' \subseteq \D,\ \phi: X\rar\pmr}\left\{\left. \frac{|\D'|}{|\D|}\ \right|\ \forall D\in \D', \dif > \tau\right\}\rp^{-1}.$$
Now, naturally, if this value is $d$ then at least $d$ queries will be needed to cover $\D$ and hence solve the problem. However, some problems might have many easy distributions making the fraction large even for hard problems. One can avoid this problem by measuring the largest covered fraction over all subsets of $\D$. Namely,
$$\SD_\dci(\B(\D,D_0),\tau) \doteq\max_{\D_0 \subseteq \D} \lp \max_{\D'\subseteq \D_0,\ \phi: X\rar\pmr}\left\{\left. \frac{|\D'|}{|\D_0|}\ \right|\ \forall D\in \D', \dif > \tau\right\}\rp^{-1}.$$
Note that $\icvr(\D,D_0,\tau)=d$ implies that $\SD_\dci(\B(\D,D_0),\tau) \leq d$. On the other hand, if $\SD_\dci(\B(\D,D_0),\tau) \leq d$ then, we can create a cover for $\D$ using the standard greedy covering algorithm: start with $\D_0 = \D$; given $\D_i$ find a function $\phi_i$ that distinguishes at least a $1/d$ fraction of distributions in $\D_i$ from $D_0$ and add it to the cover (the existence is guaranteed by the dimension); let $D_{i+1}$ be equal to $\D_i$ with the distributions covered by $\phi_i$ removed. This gives
a cover of size $d\ln(|\D|)$.
\begin{lem}
$\icvr(\D,D_0,\tau) \leq \SD_\dci(\B(\D,D_0),\tau) \cdot \ln(|\D|)$.
\end{lem}
Therefore $\SD_\dci(\B(\D,D_0),\tau)$ (which we refer to as the {\em statistical dimension with $\dci$-discrimination}) characterizes the query complexity of deterministic algorithms with access to $\STAT(\tau)$ up to a $\ln(|\D|)$ factor. 
Thus for problems where $|\D|$ is not too large (at most exponential in the relevant complexity parameters), this characterization is sufficient. We summarize this in the following corollary:
\begin{cor}
Let $\B(\D,D_0)$ be a decision problem, $\tau >0$ and $d= \SD_\dci(\B(\D,D_0),\tau)$. Then $$\QC(\B(\D,D_0),\STAT(\tau)) \geq d \mbox{ and}$$
$$\QC(\B(\D,D_0),\STAT(\tau/2)) \leq d \cdot \ln(|\D|).$$
\end{cor}

\subsection{Randomized statistical dimension}
\label{sec:stat-decision-rand}
A key notion in our tight characterization for decision problems is that of a randomized cover.
\begin{defn}
For a non-empty set of distributions $\D$ and a reference distribution $D_0$ over $X$ and $\tau >0$, let $\ircvr(\D,D_0,\tau)$ denote the smallest $d$ such that there exists a probability measure $\cQ$ over functions from $X$ to $[-1,1]$ with the property that for every $D \in \D$,
$$\pr_{\phi \sim \cQ}\lb \dif > \tau \rb \geq \fr{d}.$$
\end{defn}

We will use von Neumann's minimax theorem to show that randomized covers size can also be described as a relaxation of $\SD_\dci$ from all subsets $\D_0$ to all probability distributions over $\D$. We define these notions formally as follows. To measure the fraction of distributions in a finite set of distribution $\D$ that can be distinguished from $D_0$ we will use a probability measure\footnote{We use {\em measure} instead of a distribution to avoid confusion with input distributions.} over $\D$. That is, a function $\mu:\D\rar \R^+$ such that $\sum_{D\in\D}\mu(D) = 1$. For $\D'\subseteq \D$, we define $\mu(\D')=\sum_{D\in\D'}\mu(D)$ and recall that $S^\D$ denotes the set of probability distributions over $\D$.

\begin{defn}
For a non-empty set of distributions $\D$, a probability measure $\mu$ over $\D$, a reference distribution $D_0$ over $X$ and $\tau>0$, the maximum covered $\mu$-fraction is defined as \equn{
\dci\frc(\mu,D_0,\tau) \doteq \max_{\phi:X\rar[-1,1]} \left\{\pr_{D\sim \mu}[ \dif > \tau] \right\}.
}
\end{defn}
\begin{defn}\label{def:sdim}
  For $\tau>0$, domain $X$ and a decision problem $\B(\D,D_0)$, the \textbf{randomized statistical dimension} with $\dci$-discrimination $\tau$ of $\B(\D,D_0)$ is defined as $$\RSD_\dci(\B(\D,D_0),\tau) \doteq \sup_{\mu \in S^\D} (\dci\frc(\mu,D_0,\tau))^{-1}.$$
\end{defn}

We now show that $\RSD_\dci$ is exactly equal to the randomized cover size.
\begin{lem}
\label{lem:stat-SD-is-rcover}
For any set of distributions $\D\neq \emptyset$, a reference distribution $D_0$ over $X$ and $\tau > 0$.
\equn{\RSD_\dci(\B(\D,D_0),\tau) = \ircvr(\D,D_0,\tau).}
\end{lem}
\begin{proof}
Consider a zero-sum game in which the first player chooses a function $\phi:X\rar[-1,1]$ and the second player chooses a distribution $D\in \D$. The first player wins if $\dif > \tau$. Now the definition of $\RSD_\dci(\B(\D,D_0),\tau)=d$ states that $d$, is the lowest value such that for every probability measure $\mu$ over $\D$ there exists a function $\phi$, such that $\pr_{D \sim \mu}\lb\dif >\tau \rb \geq 1/d$ (or $1/d$ is the highest first player's  payoff).
By von Neumann's minimax theorem, $d$ is also the largest value such that for every probability measure $\cQ$ over $\pmr^X$ there exists a distribution $D \in \D$ such that $\pr_{\phi\sim \cQ}\lb\dif >\tau \rb \leq 1/d$. This is equivalent to the definition of $\ircvr(\D,D_0,\tau) = d$.
\end{proof}

We now establish that the randomized cover plays the same role for randomized algorithms as the usual cover plays for deterministic algorithms and therefore $\RSD_\dci$ tightly characterizes $\RQC$ of many-to-one decision problems.
\begin{thm}
\label{thm:random-algorithm2queries}
Let $\B(\D,D_0)$ be a decision problem, $\tau > 0, \delta \in (0,1/2)$ and $d=\RSD_\dci(\B(\D,D_0),\tau)$.
Then $$\RQC(\B(\D,D_0),\STAT(\tau),1-\delta) \geq d \cdot (1-2\delta)\mbox{ and}$$
$$\RQC(\B(\D,D_0),\STAT(\tau/2),1-\delta) \leq d \cdot \ln(1/\delta).$$
\end{thm}
\begin{proof}
Let $\A$ be the algorithm that solves $\B(\D,D_0)$ with probability $1-\delta$ using $q$ queries to $\STAT(\tau)$.  We simulate $\A$ by answering any query $\phi:X \rightarrow [-1,1]$ of $\A$ with value $D_0[\phi]$. Let $\phi_1,\phi_2,\ldots,\phi_q$ be the queries asked by $\A$ in this simulation (note that the queries are random variables that depend on the randomness of $\A$). Now let $D$ be any distribution in $\D$ and define
$$p_D \doteq \pr_\A\lb\exists i\in[q],\ \difp{\phi_i} > \tau\rb .$$
If $p_D <1-2\delta$ then, with probability $> 2\delta$, all the responses in our simulation are valid responses of $\STAT_D(\tau)$. By the correctness of $\A$, $\A$ can output ``$D\in \D$" with probability at most $\delta$ in this simulation. This means that for some valid answers of $\STAT_D(\tau)$ for $D \in \D$, with probability $>2\delta -\delta = \delta$ the algorithm will output $``D =D_0"$ contradicting our assumption. Hence $p_D\geq 1-2\delta$ and for every $D$, with probability at least $1-2\delta$, there exists $i$, such $\phi_i$ generated by $\A$ in this (fixed) simulation distinguishes between $D$ and $D_0$. Therefore taking $\cQ$ to be the distribution obtained by running $\A$ and then picking one of its $q$ queries randomly and uniformly ensures that $$\pr_{\phi \sim \cQ}\lb \dif > \tau\rb \geq \frac{1-2\delta}{q} .$$ This proves that $\ircvr(\D,D_0,\tau) \leq q/(1-2\delta)$.

For the other direction: let $\cQ$ be the probability measure over functions such that
$$\pr_{\phi \sim \cQ}\lb \dif > \tau\rb \geq \frac{1}{d} .$$
For $s=d\ln(1/\delta)$ we sample $s$ functions from $\cQ$ randomly and independently and denote them by $\phi_1,\ldots,\phi_s$. For every $i\in [s]$ we ask the query $\phi_i$ to $\STAT(\tau/2)$ and let $v_i$ be the response. If exists $i$ such that $|v_i - D_0[\phi_i]| > \tau/2$ then we conclude that the input distribution is not $D_0$. Otherwise, we output that the input distribution is $D_0$. By the definition of $\STAT(\tau/2)$, this algorithm will always be correct when $D=D_0$. Further, for every $D\in \D$, by eq.~\eqref{eq:cover-high-prob} we have that with probability at least $1-\delta$, for some $i$, $\difp{\phi_i} > \tau$, which implies that $|v_i - D_0[\phi_i]| > \tau/2$.
This ensures that the response of our algorithm will be correct with probability at least $1-\delta$ for all distributions in $\D$.
\end{proof}

\paragraph{Relationship to $\QC$:} We conclude this section by comparing the notions we have introduced with those used in Sec.~\ref{sec:stat-decision-det} to characterize $\QC$. First, by taking $\cQ$ to be the uniform distribution over the functions that give the deterministic $\tau$-cover we immediately get that
\equ{\ircvr(\D,D_0,\tau) \leq \icvr(\D,D_0,\tau). \label{eq:cvr-2-rcvr}}
We also observe that a randomized cover can bee easily converted into a deterministic one (see Lemma \ref{lem:stat-rand-to-det-cover} for the proof):
\equ{\icvr(\D,D_0,\tau) \leq \ircvr(\D,D_0,\tau) \cdot \ln(|\D|) . \label{eq:rcvr-2-cvr}}

By restricting $\mu$ in the definition of $\RSD_\dci(\B(\D,D_0),\tau)$ (Def.~\ref{def:sdim}) to be any measure that is uniform over some $\D_0 \subseteq \D$, we obtain precisely the definition of $\SD_\dci(\B(\D,D_0),\tau)$. Thus, $\RSD_\dci(\B(\D,D_0),\tau) \geq \SD_\dci(\B(\D,D_0),\tau)$. This is in contrast to the opposite relationship between the randomized and deterministic complexity (such as the one given in eq.~\eqref{eq:cvr-2-rcvr}). Hence $\RSD_\dci(\B(\D,D_0),\tau)$ is closer to the deterministic SQ complexity of $\B(\D,D_0)$ than $\SD_\dci(\B(\D,D_0),\tau)$. At the same time both $\SD_\dci(\B(\D,D_0),\tau)$ and $\RSD_\dci(\B(\D,D_0),\tau)$ characterize $\QC(\B(\D,D_0),\STAT(\tau))$ up to a factor of $\ln(|\D|)$. A natural open problem would be to find an easy to analyze (in particular, one that does not rely on $\icvr$) characterization for deterministic algorithms that avoids this factor.

\section{Characterization for general search problems}
\label{sec:stat-search}
We now extend our statistical dimension to the general class of search problems. We characterize the statistical dimension using the statistical dimension of the hardest many-to-one decision problem associated with the search problem. Naturally, this is a standard approach for proving lower bounds and our lower bound follows easily from those for decision problems. On the other hand, the fact that the converse (or upper bound) holds is substantially more remarkable and relies crucially on the properties of statistical queries.

\subsection{Deterministic dimension for search problems}
\label{sec:stat-search-det}
We now describe the statistical dimension for general search problems. We will first deal with the simpler deterministic characterization and also use it to introduce the key idea of our approach. In Section \ref{sec:stat-search-rand} we will show how the dimension needs to be modified to obtain a general characterization for randomized algorithms.

\begin{defn}\label{def:sdim1-search}
  For $\tau>0$, domain $X$ and a search problem $\Z$ over a set of solutions $\F$
  and a class of distributions $\D$ over $X$,
  we define the \textbf{statistical dimension} with $\dci$-discrimination $\tau$ of $\Z$ as
  $$\SD_\dci(\Z,\tau) \doteq \sup_{D_0 \in S^X} \inf_{f\in \F} \RSD_\dci(\B(\D\setminus \Zf,D_0),\tau),$$
  where $S^X$ denotes the set of all probability distributions over $X$.
\end{defn}
\iffull
\fi
Now the proof of the lower bound is just a reduction to the argument we used for the decision problem case.
\begin{thm}
\label{thm:stat-search-lower}
  For any search problem $\Z$ and $\tau >0$, $\QC(\Z,\STAT(\tau)) \geq \SD_\dci(\Z,\tau)$.
\end{thm}
\iffull
\begin{proof}
   Let $\A$ be a deterministic statistical algorithm that uses $q$ queries to  $\STAT(\tau)$ to solve $\Z$.
   By the definition of $d \doteq \SD_\dci(\Z,\tau)$, for any $d' <d$, there exists a distribution $D_0$ over $X$ such that for every $f\in \F$,
    $\RSD_\dci(\B(\D\setminus \Zf,D_0),\tau) \geq d'$.

    We simulate $\A$ by answering any query $\phi:X \rightarrow [-1,1]$ of $\A$ with value $D_0[\phi]$. Let $\phi_1,\ldots,\phi_q$ be the queries generated by $\A$ in this simulation and let $f_0$ be the  output of $\A$. By the correctness of $\A$, we know that for every $D\in \D$ for which $f_0$ is not a valid solution, the answers based on $D_0$ cannot  be valid answers of $\STAT_D(\tau)$. In other words, for every $D\in \D\setminus \Z_{f_0}$, there exists $i\in[q]$ such that $|D[\phi_i] - D_0[\phi_i]| > \tau$. This implies that $\icvr(\D\setminus \Z_{f_0},D_0,\tau) \leq q$. By eq.~\eqref{eq:cvr-2-rcvr} and Lemma \ref{lem:stat-SD-is-rcover}, we have that $$\RSD_\dci(\B(\D\setminus \Z_{f_0},D_0),\tau) \leq \icvr(\D\setminus \Z_{f_0},D_0,\tau) \leq q.$$ Combining this with $\RSD_\dci(\B(\D\setminus \Z_{f_0},D_0),\tau) \geq d'$ we get that $q \geq d'$. The claim holds for any $d'<d$ implying the statement of the theorem.
\end{proof}
\fi

The proof of the upper bound relies on the well-known Multiplicative Weights algorithm. Specifically, we will use the following result that was first proved in the classic work of \citet{Littlestone:87}. Our presentation and specific bounds are based on a more recent view of the algorithm in the framework of online convex optimization (\eg \cite{AroraHK12,Shalev-Shwartz12}).
For a positive integer $m$, let $S^m$ be the $m$-dimensional simplex $S^m  \doteq \{w \cond \|w\|_1 = 1,\ \forall_{i\in[m]} w_i\geq 0\}$.
\begin{figure}[h]
\begin{boxedminipage}{3.5in}
 \textbf{Multiplicative Weights (MW)}\\
\textbf{Input:} $\gamma > 0$, $w^1 \in S^m$\\
\textbf{Update at step $t$:}
 Given a linear loss function $z^t \in \pmr^m$:\\
1.\ For all $i\in [m]$, set $\hat{w}^{t+1}_i = w^{t}_i (1-\gamma z_i)$;\\
2.\ Set $w^{t+1} = \hat{w}^{t+1}/\|\hat{w}^{t+1}\|_1$.
\end{boxedminipage}
 \caption{Online linear optimization via Multiplicative Weights}
\label{fig:mwu}
\end{figure}
\begin{thm}
\label{thm:mwu}
For any sequence of loss vectors $z^1,\ldots,z^T \in \pmr^m$, Multiplicative Weights algorithm (Fig.\ref{fig:mwu}) with input $\gamma$ and $w^1$ produces a sequence weight vectors $w^1,\ldots,w^T$, such that for all $w \in S^m$ 
$$\sum_{t\in [T]} \la w^t, z^t \ra - \sum_{t\in [T]} \la w, z^t \ra \leq  \frac{\KL(w \| w^1)}{\gamma} + \gamma T ,$$ where $\KL(w \| w^1) \doteq \sum_{i\in [m]} w_i \ln(w_i/w_i^1)$. Thus for $T \geq \frac{4 \cdot \KL(w \| w^1)}{\gamma^2}$,  the average regret is at most $\gamma$.
\end{thm}
In our setting the weight vectors correspond to probability distributions over some finite domain $X$ and linear loss functions correspond to statistical query functions. Interpreted in this way we obtain the following result.
\begin{cor}
\label{cor:mwu-sq}
Let $X$ be any finite domain and $\gamma >0$. Consider an execution of the MW algorithm with parameter $\gamma$ and initial distribution $D_1$ on a sequence of functions $\psi_1,\ldots,\psi_T:X\rar \pmr$ and let $D_1,\ldots,D_T$ be the sequence of distributions that was produced. Then for every distribution $D$ over $X$ and $T \geq \frac{4\cdot\KL(D \| D^1)}{\gamma^2}$ we have
$$\frac{1}{T} \cdot \sum_{t\in [T]} \lp D_t[\psi_t] - D[\psi_t] \rp \leq \gamma .$$
\end{cor}

Now we can describe the upper bound. We will express it in terms of the radius of the set of all distributions $\D$ measured in terms of KL-divergence. Namely, we define
\equ{\KLR(\D) \doteq \min_{D_1 \in S^X } \max_{D\in \D} \KL(D\|D_1). \label{eq:kl-radius-def} }
We observe that $\KLR(\D) \leq \ln(|\D|)$ by taking $D_1 \doteq \frac{1}{|\D|} \sum_{D' \in \D} D'$ to be the uniform combination of distributions in $\D$. We also note that $\KLR(\D) \leq \ln(|X|)$ by taking $D_1$ to be the uniform distribution over $X$. In many search problems it could be much smaller. For example, in distribution-specific learning it is at most $\ln 2$.
\begin{thm}
\label{thm:stat-search-upper}
  For any search problem $\Z$, over a finite class of distributions $\D$ on a finite domain $X$ and $\tau >0$:
    $$\QC(\Z,\STAT(\tau/3)) = O(\SD_\dci(\Z,\tau) \cdot  \log |\D| \cdot \KLR(\D)/\tau^2).$$
\end{thm}
\begin{proof}
The key idea of the proof is that ability to distinguish any reference distribution from the input distribution using a query can be used to reconstruct the input distribution $D$ via the multiplicative weights update algorithm. If we fail to distinguish the input distribution from the reference distribution then we can find a valid solution.

Formally, we start with $D_1$ that minimizes the $\KLR(\D)$ as defined in eq.~\eqref{eq:kl-radius-def}. Let $D_t$ denote the distribution at step $t$. By the definition of $d\doteq \SD_\dci(\Z,\tau)$, there exists $f \in \F$ such that $\RSD_\dci(\B(\D\setminus \Z_f,D_t),\tau) \leq d$. By eq.~\eqref{eq:rcvr-2-cvr} we get that  $\icvr(\D\setminus \Z_f,D_t,\tau) \leq d\ln(|\D\sm\Z_f|)$. Let $\phi_1,\ldots,\phi_s$ for $s = \icvr(\D\setminus \Z_f,D_t,\tau)$ be a $1$-cover of $\D\setminus \Z_f$ with tolerance $\tau$. For every $i\in[s]$, we make query $\phi_i$ to $\STAT(\tau/3)$ and let $v_i$ denote the response. If there exists $i$ such that $\left|D_t[\phi_i] - v_i\right| >2\tau/3$, then we define $\psi_t \doteq \phi_i$ if $D_t[\phi_i] > v_i$ and $\psi_t \doteq - \phi_i$, otherwise. We then define $D_{t+1}$ using the update of the MW algorithm on $\psi_t$ with $\gamma=\tau/3$ and go to the next step.  Otherwise (if no such $\phi_i$ exists), we output $f$ as the solution.

We first establish the bounds on the complexity of the algorithm. By the correctness of $\STAT(\tau/3)$ we have that for every update step \equ{\left|D_t[\phi_i] - D[\phi_i]\right| > \frac{2\tau}{3}-\frac{\tau}{3} = \frac{\tau}{3}.\label{eq:mw-updatestrength}} As a consequence, $D_t[\psi_t] - D[\psi_t] >\tau/3$. By Cor.~\ref{cor:mwu-sq}, this implies that there can be at most $T \leq \frac{36\cdot \KL(D \| D_1)}{\tau^2} \leq  \frac{36\ln(|\D|)}{\tau^2}$ such updates. Using the bound on the number of queries in each step, we immediately get the stated bounds on the complexity of the algorithm.

To establish the correctness, we note that at every step, for every $D \in \D\setminus \Z_f$ we are guaranteed to perform an update since there exists a function $\phi_i$ in the cover such that $\left|D_t[\phi_i] - D[\phi_i]\right| >\tau$. This means that we only output a solution $f$ when $D \in \Z_f$, which is exactly the definition of correctness.
\end{proof}

\iffull
\begin{remark}
To simplify the upper-bound we can always replace $\KLR(\D)$ with $\log(|\D|)$ since we already have one such term from eq.~\eqref{eq:rcvr-2-cvr}. 
\end{remark}

\begin{remark}
\label{rem:mw-project}
We can also ensure that the sequence of distributions produced by MW stays within the convex hull of distributions in $\D$ (which we denote by $\conv(\D)$). This can be achieved by performing a projection to $\conv(\D)$ that minimizes KL-divergence (see \cite{AroraHK12} for details). This implies that for the upper bound (and characterization) it is sufficient to have an upper bound on $$\sup_{D_0 \in \conv(\D)} \inf_{f\in \F} \RSD_\dci(\B(\D\setminus \Zf,D_0),\tau).$$ Alternatively, the same effect can be achieved by performing the multiplicative updates on $\conv(\D)$ viewed as a $|\D|$-dimensional simplex of coefficients representing a distribution in $\conv(\D)$. In this case the updates will use the vector $\left(D[\psi_t]\right)_{D \in \D}$ instead of $\psi_t$ itself.
\end{remark}
\fi

\subsection{Randomized dimension for search problems}
\label{sec:stat-search-rand}
To prove lower bounds against randomized SQ algorithms we need a stronger notion that we define below. The main issue is that in the randomized setting the interplay between distribution over queries, distribution over solutions and success probability can be rather complex. In particular, the way that success probability affects the complexity depends strongly on the type of problem. For example, in general decision problems (not just many-vs-one that we already analyzed) only success probability above $1/2$ can have non-trivial complexity. On the other hand, in search problems with (exponentially) large search space the SQ complexity is often high for any inverse polynomial probability of success. To reflect such dependence we parameterize the randomized SQ dimension by success probability $\alpha$.

\begin{defn}\label{def:rsdim1-search}
  Let $\Z$ be a search problem over a set of solutions $\F$ and a class of distributions $\D$ over a domain $X$ and let $\tau>0$.
  For a probability measure $\cP$ over $\F$ and $\alpha > 0$, we denote by $\Z_\cP(\alpha) \doteq \left\{D \in \D \cond \cP(\Z(D)) \geq \alpha \right\}$.
  For a success probability parameter $\alpha$, we define the \textbf{randomized statistical dimension} with $\dci$-discrimination $\tau$ of $\Z$ as
  $$\RSD_\dci(\Z,\tau,\alpha) \doteq \sup_{D_0 \in S^X}\inf_{\cP \in S^\F} \RSD_\dci(\B(\D\setminus \Z_\cP(\alpha),D_0),\tau) .$$
\end{defn}

For $\alpha =1$, $\Z_\cP(\alpha)$ is equal to the intersection of $\Zf$ for $f$ in the support of $\cP$. This set is maximized (and consequently $\RSD_\dci(\B(\D\setminus \Z_\cP(\alpha),D_0),\tau)$ is minimized) when the support of $\cP$ is just a single element. Hence $\RSD_\dci(\Z,\tau,1) = \SD_\dci(\Z,\tau)$ implying that $\RSD$ is a generalization of $\SD$.

We can now prove a lower bound against randomized algorithms using an approach similar to the one we used for decision problems.

\begin{thm}
\label{thm:stat-search-lower-random}
    For any search problem $\Z$, $\tau >0$ and $\beta > \alpha > 0$,
     $$\RQC(\Z,\STAT(\tau),\beta) \geq \RSD_\dci(\Z,\tau,\alpha) \cdot (\beta - \alpha).$$
\end{thm}
\begin{proof}
Let $\A$ be the algorithm that solves $\Z$ with probability $\beta$ using $q$ queries to $\STAT(\tau)$.
By the definition of $d\doteq \RSD_\dci(\Z,\tau,\alpha)$, for any $d' <d$, there exists a distribution $D_0$ over $X$ such that for every $\cP \in S^\F$,   $\RSD_\dci(\B(\D\setminus\Z_\cP(\alpha),D_0),\tau) \geq d'$.

We simulate $\A$ by answering any query $\phi:X \rightarrow [-1,1]$ of $\A$ with value $D_0[\phi]$. Let $\phi_1,\phi_2,\ldots,\phi_q$ be the queries asked by $\A$ in this simulation and let $f_0$ denote the solution produced (note that the queries and the solution are random variables that depend on the randomness of $\A$). Now let $D$ be any distribution in $\D$ and define
$$p_D \doteq \pr_\A\lb\exists i\in[q],\ \difp{\phi_i} > \tau\rb .$$
Let $\cP_0$ denote the PDF of $f_0$. If $D \not\in \Z_{\cP_0}(\alpha)$ then $\pr_\A[f_0 \in \Z(D)] < \alpha$.
This implies that $p_D \geq \beta -\alpha$ since with probability $1-p_D$, all the responses in our simulation are valid responses for $\STAT_D(\tau)$. The algorithm $\A$ fails with probability at most $1-\beta$ and therefore it has to output $f_0 \in \Z(D)$ with probability at least $1-p_D - (1-\beta)$. By our assumption, this probability is less than $\alpha$ and therefore $p_D > \beta -\alpha$.

Now, taking $\cQ$ to be the uniform distribution over $\phi_1,\phi_2,\ldots,\phi_q$ ensures that $$\pr_{\phi \sim \cQ}\lb \dif > \tau\rb \geq \frac{p_D}{q} > \frac{\beta - \alpha}{q} .$$ This proves that $\ircvr(\D \sm \Z_\cP(\alpha) ,D_0,\tau) < q/(\beta -\alpha)$. By Lemma \ref{lem:stat-SD-is-rcover}, $\RSD_\dci(\B(\D\setminus \Z_{\cP_0}(\alpha),D_0),\tau)  < q/(\beta -\alpha)$ and thus
$q > d'(\beta - \alpha)$. This holds for every $d' < d$ implying the claim.
\end{proof}

\iffull
\citet{FeldmanGRVX:12} describe a somewhat different way to reduce search problems to decision problems in the randomized case. Their approach is based on upper bounding the fraction of distributions for which any solution can be valid. To the best of our knowledge, this approach does not lead to a characterization for search problems. For comparison, we describe how this approach leads to lower bounds in Lemma \ref{lem:stat-search-lower-random-simple}.
\fi

\remove{
Now observe that if $d= \RSD_\dci(\Z,\tau,\alpha)$ then there exists a reference distribution $D_0$ and a measure $\mu$ over $\D$ such that $\mu(\F) \leq \alpha$ and $\left( \dci\frc(\mu,D_0,\tau)\right)^{-1} \geq d$. Plugging this into Lemma \ref{lem:stat-search-lower-random}, and also using the Yao's minimax principle \cite{Yao:1977} we get a lower bound of $(\beta - \alpha)d$ queries for randomized algorithms.
}

We now demonstrate that $\RSD_\dci(\Z,\tau,\alpha)$ can be used to upper bound the SQ complexity of solving $\Z$ with success probability almost $\alpha$.
The proof is based on a combination of the analysis we used in the deterministic upper bound for search problems with the use of dual random sampling algorithm as in the randomized upper bound for decision problems. At each step of the MW algorithm, the definition of $\RSD_\dci$ guarantees a randomized cover only for a subset of input distributions. At the same time, the definition guarantees that there exists a fixed distribution over solutions that gives, with probability at least $\alpha$, a valid solution for every input distribution that is not covered.
\begin{thm}
\label{thm:stat-search-upper-random}
For any search problem $\Z$  over a finite class of distributions $\D$ on a finite domain $X$, $\tau >0$ and $\alpha > \delta > 0$,
     $$\RQC(\Z,\STAT(\tau/3),\alpha-\delta) = O\lp  \RSD_\dci(\Z,\tau,\alpha) \cdot \frac{\KLR(\D)}{\tau^2} \cdot\log\lp\frac{ \KLR(\D)}{\tau\delta}\rp\rp .$$
\end{thm}
\begin{proof}
We set the initial reference distribution $D_1$ to be $D_1$ that minimizes the $\KLR(\D)$ as defined in eq.~\eqref{eq:kl-radius-def}
and initialize $\F' = \emptyset$.
 Let $T \doteq \frac{36\cdot\KLR(\D)}{\tau^2}$ and $\delta' \doteq \delta/T$.
Let $D_t$ be the current reference distribution.

By Definition \ref{def:rsdim1-search}, $d \doteq \RSD_\dci(\Z,\tau,\alpha)$ implies that there exists a probability measure $\cP_t$ over $\F$
such that $\RSD_\dci(\B(\D\setminus \Z_{\cP_t}(\alpha),D_0),\tau) \leq d$. This, by Lemma \ref{lem:stat-SD-is-rcover}, implies that there exists a measure $\cQ$ over $\pmr^X$ such that for all $D \in \D\setminus \Z_{\cP_t}(\alpha)$, \equ{\pr_{\phi \sim \cQ}\lb \diff{\phi}{D}{D_t}>\tau \rb \geq \fr{d} .\label{eq:good-p}}

For $s = d\ln(1/\delta')$ we draw $s$ independent samples from $\cQ$ and denote them by $\phi_1,\ldots,\phi_s$. For every $i\in[s]$ we make query $\phi_i$ to $\STAT(\tau/3)$. Let $v_i$ denote the response. If there exists $i$ such that $\left|D_t[\phi_i] - v_i\right| > 2\tau/3$ then we define $\psi_t \doteq \phi_i$ if $D_t[\phi_i] > v_i$ and $\psi_t \doteq - \phi_i$, otherwise. We then define $D_{t+1}$ using the update of the MW algorithm on $\psi_t$ with $\gamma=\tau/3$ and go to the next step. Otherwise (if no such $\phi_i$ exists), we choose $f$ randomly according to $\cP_t$ and output it.

We first establish the bounds on the complexity of the algorithm. As in the proof of Theorem \ref{thm:stat-search-upper}, we get that there are at most $\frac{36\cdot \KL(D \| D_1)}{\tau^2}$ update steps. Given our definition of $D_1$ we get an upper bound of $\frac{36\cdot\KLR(\D)}{\tau^2} = T$.
Using the bound on samples in each step, we immediately get the stated bounds on the SQ complexity of the algorithm.

To establish the correctness observe that if at the last step $D \in \Z_{\cP_t}(\alpha)$ then with probability at least $\alpha$ we output $f \in \Z(D)$. This condition is not satisfied only if at some step $t$ we do not make an update even though $D \in \D \sm \Z_{\cP_t}(\alpha)$. By eq.~\eqref{eq:good-p} this happens with probability
\equn{\pr_{(\phi_1,\ldots,\phi_s)\sim \cQ^s}\lb\forall i\in[s],\ \diff{\phi_i}{D}{D_t}\leq \tau \rb \leq \lp 1-\fr{d}\rp^s \leq \delta'  .} Therefore the total probability of this condition ($D \in \Z_{\cP_t}(\alpha)$ at the last step)  is at most $T\delta' = \delta$. Hence the probability of success of our algorithm is at least $\alpha -\delta$.
\end{proof}

\subsection{Special cases: optimizing  and verifiable search}
\label{sec:search-special}
We now show how our characterization can be simplified for verifiable and optimizing search problems (see Sec.~\ref{sec:prelims} for the definition and examples of such problems).
First, recall that in a verifiable search problem, for every $f\in \F$, there is an associated query function $\phi_f: X\rar [0,1]$ such that the search problem $\V$ with parameter $\theta$ is defined as
$$\V_\theta(D) \doteq \left\{f\ \left|\ D[\phi_f] \leq \theta \right.\right\}.$$

To avoid dealing with success probability due to finding a solution we can instead avoid reference distributions for which any solution can pass the verification step. Namely, we define $$\D_\theta \doteq \left\{D \in S^X\ \left|\ \forall f \in \F,\ D[\phi_f] > \theta \right.\right\},$$ or equivalently,
$D\in \D \sm \D_\theta$ if and only if $\V_\theta(D) \neq \emptyset$.
 We then define the following statistical dimension:
\begin{defn}\label{def:sdim-verifiable}
  For $\theta \geq 0$, let $\V$ be a verifiable search problem with parameter $\theta$ over a set of solutions $\F$ and a class of distributions $\D$ over a domain $X$ and let $\tau>0$.
    We define the \textbf{randomized statistical dimension} with $\dci$-discrimination $\tau$ of $\V_\theta$ as
  $$\RSDV_\dci(\V_\theta,\tau) \doteq \sup_{D_0 \in \D_\theta} \RSD_\dci(\B(\D,D_0),\tau) .$$
\end{defn}
This definition gives the following characterization:
\begin{thm}
\label{thm:stat-verifiable-lower}
  Let $\V$ be a verifiable search problem over a class of distributions $\D$. For any $\theta \geq \tau > 0, \beta>0$,
  $$\RQC(\V_{\theta - \tau},\STAT(\tau), \beta) \geq \beta \cdot \RSDV_\dci(\V_\theta,\tau) -1. $$
\end{thm}
\begin{proof}
Let $\A$ be an algorithm that solves $\V_{\theta - \tau}$. Let $\A'$ be an algorithm that runs $\A$ then, given a solution $f$ output by $\A$ asks query $\phi_f$ to $\STAT(\tau)$. If the answer $v > \theta$ then it fails (say outputs $\bot \not\in \F$). Clearly, $\A'$ solves $\V_{\theta - \tau}$ with the same success probability $\beta$ as $\A$. We now apply the analysis from Thm.~\ref{thm:stat-search-lower-random} with $D_0 \in \D_\theta$ to $\A'$. By definition of $\D_\theta$, we know that for every $f\in \F$, $D_0[\phi_f] > \theta$. Therefore the value $v$ that $\A'$ gets in our simulation to its last verification query satisfies $v > \theta$ and thus the algorithm will fail. This means that
for every distribution $D$, $\A'$ is successful with probability at least $\beta$. From here the analysis is identical to that in Thm.~\ref{thm:stat-search-lower-random}.
\end{proof}

\begin{thm}
\label{thm:stat-verifiable-upper}
   Let $\V$ be a verifiable search problem over a class of distributions $\D$. For any $\theta\geq 0, \tau > 0, \delta > 0$ $$\RQC(\V_{\theta + \tau},\STAT(\tau/3),1-\delta) = \tilde{O}\lp \RSD_\dci(\V_\theta,\tau) \cdot \frac{\KLR(\D)}{\tau^2} \cdot\log(1/\delta)\rp .$$
\end{thm}
\begin{proof}
We perform the same basic algorithm as in Thm.~\ref{thm:stat-search-upper-random}. Note that as long as $D_t \not\in \D_\theta$ the characterization guaranteed that we will find a distinguishing query with probability at leat $1-\delta'$ for all $D\in \D$. If we reach $D_t \in \D_\theta$, then we know that there exists a function $\phi_f:X\rar [0,1]$ such that $D_t[\phi_f] \leq \theta$. We ask the query $\phi_f$ to $\STAT(\tau/3)$. If the response $v \leq \theta + 2\tau/3$ then we output $f$ as the solution and stop. Note that this implies that $D[\phi_f] \leq \theta+\tau$ and therefore $f$ is a valid solution to $\V_{\theta +\tau}$. Otherwise, we update the distribution using $\psi_t = - \phi_f$ and go to the next step. In this case $D[\phi_f]  \geq \theta + \tau/3$ and hence $D_t[\psi_t] -  D[\psi_t] \geq \tau/3$. Therefore the same bound on the number of iterations applies and the number of queries grows just by one in every round.
\end{proof}
In most settings, verifiable search with parameter $\theta$ requires $\tau < \theta/2$ and therefore we get a characterization up to at most constant factor increase in the threshold $\theta$.

We now deal with linear optimizing search problems. Recall that in a linear optimizing search problem every $f \in \F$ is associated with a function $\phi_f:X\rar [0,1]$ and for $\eps > 0$, $$\Z_\eps(D) \doteq \left\{h\ \left|\ D[\phi_h] \leq \min_{f \in \F}\{D[\phi_f]\} +\eps\right.\right\}.$$
Solving an $\eps$-optimizing linear search problem is essentially equivalent to solving the range of associated verifiable search problems. Therefore we characterize $\eps$-optimizing search using the statistical dimension of verifiable search problems.
\begin{defn}\label{def:sdim-opt-via-verifiable}
Let $\Z$ be a linear optimizing search over a class of distributions $\D$ and set of solutions $\F$.
    For $\eps>0$, we define the \textbf{randomized statistical dimension} with $\dci$-discrimination $\tau$ of $\Z_\eps$ as
  $$\RSDV_\dci(\Z_\eps,\tau) \doteq  \sup_{\theta \in [0,1],\ D_0 \in \D_{\theta+\eps}} \RSD_\dci(\B(\D \sm \D_\theta,D_0),\tau) .$$
\end{defn}

For every distribution $D\in \D \sm \D_\theta$, there exists a solution $f$ such that $D[\phi_f]\leq \theta$. This means that solving $\Z_\eps$ restricted to $\D \sm \D_\theta$ requires finding $h\in \F$ such that $D[\phi_h] \leq \theta+\eps$. This means that by using our lower bound for $\V_{\theta+\eps}$ we get the following lower bound for $\Z_\eps$.
\begin{thm}
\label{thm:stat-optimization-via-verification-lower}
  Let $\Z$ be a linear optimizing  search problem over a class of distributions $\D$. For any $\eps \geq \tau > 0, \beta>0$,
  $$\RQC(\Z_{\eps - \tau},\STAT(\tau), \beta) \geq \beta \cdot \RSDV_\dci(\Z_\eps,\tau) -1. $$
\end{thm}

In the opposite direction: If we can solve the verifiable search problem for every $\theta$ then, by using binary search, we can find the minimum (up to $\tau/4$) $\theta$ for which there is a (verifiable) solution. This increases the complexity of the algorithm by a factor of $\log(4/\tau)$. Now using Theorem \ref{thm:stat-verifiable-upper} with tolerance $3\tau/4$ we obtain:
\begin{thm}
\label{thm:stat-optimization-via-verification-upper}
  Let $\Z$ be a linear optimizing  search problem over a class of distributions $\D$. For any $\eps \geq \tau > 0, \delta>0$,
 $$\RQC(\Z_{\eps + \tau},\STAT(\tau/4),1-\delta) = \tilde{O}\lp \RSD_\dci(\Z_\eps,\tau) \cdot \frac{\KLR(\D)}{\tau^2} \cdot\log(1/\delta)\rp .$$
\end{thm}

\remove{
\begin{remark}
\label{rem:improve-mw-bound}
Note that we can easily combine the analysis of Thm.~\ref{thm:stat-decision-upper} with the MW analysis in this proof. Namely, we can restrict the analysis to the set of distributions $\D_i$ that have not been eliminated by one of the queries. Then each query either eliminates a $\tau/(6d)$ fraction of $\D_i\setminus \Z^{-1}(f_t)$ or leads to a MW update. This can be used to improve the bounds in several ways. For example, the number of elimination updates cannot exceed $|\F| \cdot 6d\ln(|\D|)/\tau$ since each time at least a fraction $\tau/(6d)$ of distributions incompatible with one of the solutions is eliminated. Therefore we will get a total upper bound of $O(\max\{d|\F|,1/\tau\} \cdot \log(|\D|) /\tau)$. This bound can be a substantial improvement for problems with small solution set (\eg distribution testing).
\end{remark}
}

\section{Characterizing the power of $\VSTAT$}
\label{sec:vstat-power}
Our main goal is to accurately characterize the power of the more involved (but also more faithful) $\VSTAT$ oracle. Unfortunately, the discrimination operator that corresponds to $\VSTAT$ is rather inconvenient to analyze directly. In particular, unlike $\STAT$, it is not symmetric for the purposes of distinguishing between two distributions. That is, if $p = D[\phi]$ and $p_0=D_0[\phi]$ then $p$ can be a valid answer of $\VSTAT_{D_0}(n)$ to $\phi$ whereas $p_0$ is not a valid answer of $\VSTAT_{D}(n)$. We show that the analysis can be greatly simplified by introducing an oracle that is equivalent (up to a factor of 3) to $\VSTAT$ while being symmetric. As a result, it behaves almost in the same way in our characterization. This allows us to directly map results from Sections \ref{sec:stat-decision} and  \ref{sec:stat-search} to $\VSTAT$.

\begin{defn}
For $\tau > 0$ and distribution $D$, a statistical query oracle $\vSTAT_D(\tau)$ is an oracle that given as input any function $\phi : X \rightarrow [0,1]$ returns a value $v$ such that $\left|\sqrt{v} - \sqrt{D[\phi]}\right|\leq \tau$.
\end{defn}
We prove the following equivalence between $\vSTAT$ and $\VSTAT$.
\begin{lem}
\label{lem:vstat-reduction}
Any query $\phi:X\rar [0,1]$ to $\vSTAT_D(\tau)$ can be answered using the response to a query $\phi$ for $\VSTAT_D(1/\tau^2)$. Any query $\phi:X\rar [0,1]$ to $\VSTAT_D(n)$ can be answered using the response to a query $\phi$ for $\vSTAT_D(1/(3\sqrt{n}))$.
\end{lem}
Note that when $p\leq 1/2$, $\sqrt{p}$ is equal (up to a factor of 2) to the standard deviation of the Bernoulli random variable with bias $p$ (which is $\sqrt{p(1-p)}$). Therefore this equivalence implies the following additional interpretation for the accuracy of $\VSTAT$. It returns any value $v$ as long as the standard deviation of the Bernoulli random variable with bias $v$ differs by at most $\frac{1}{\sqrt{n}}$ (up to constant factors) from the standard deviation of the Bernoulli random variable with bias $p$. Making this statement precise would require defining $\vSTAT(\tau)$ as returning $v$ such that $|\sqrt{v(1-v)} - \sqrt{p(1-p)}| \leq \tau$. As in the case of $\VSTAT$ (which we discussed in Remark \ref{rem:vstat-simplify}), this version is equivalent (up to a factor of two) to our simpler definition.

\begin{lem}
\label{lem:vstat2root}
For any $p,\tau \in [0,1]$, let $v\in [0,1]$ be any value such that $|v-p| \leq \max\left\{\tau^2,\sqrt{p}\tau\right\}$. Then $|\sqrt{v}-\sqrt{p}| \leq \tau$.

For any $p,\tau \in [0,1]$, let $v\in[0,1]$ be any value such that $|\sqrt{v}-\sqrt{p}| \leq \tau/3$. Then $|v-p| \leq \max\{\tau^2,\sqrt{p}\tau\}$.
\end{lem}
\begin{proof}
First part: Assuming for the sake of contradiction that $|\sqrt{v}-\sqrt{p}| > \tau$ we get $$|v-p| = |\sqrt{v}-\sqrt{p}| \cdot (\sqrt{v}+\sqrt{p}) > (\sqrt{v}-\sqrt{p})^2 > \tau^2 $$ and
$$|v-p| = |\sqrt{v}-\sqrt{p}| \cdot (\sqrt{v}+\sqrt{p}) > \tau \cdot \sqrt{p} .$$ This contradicts the definition of $v$.

Second part: We first note that it is sufficient to prove this statement when $v >p$ (the other case can be obtained by swapping the values of $p$ and $v$). Next, observe that it is sufficient to prove that $\sqrt{v}-\sqrt{p} \leq \max\{3\tau,3\sqrt{p}\}$ since then we will get that
$$ |v-p| = (\sqrt{v}-\sqrt{p}) \cdot (\sqrt{v}+\sqrt{p}) \leq \frac{\tau}{3} \cdot \max\{3\tau,3\sqrt{p}\} = \max\{\tau^2,\sqrt{p}\tau\}.$$
To prove that $\sqrt{v}-\sqrt{p} \leq \max\{3\tau,3\sqrt{p}\}$ we consider two cases. If $\sqrt{v} \leq 2\sqrt{p}$ then clearly $\sqrt{v}-\sqrt{p} \leq 3\sqrt{p}$. Otherwise, if $\sqrt{v} > 2\sqrt{p}$. Then we get that $\frac{\tau}{3} \geq \sqrt{v} - \sqrt{p} \geq \sqrt{v} - \sqrt{v}/2 = \sqrt{v}/2$ or $\sqrt{v} \leq 2\tau/3$. This implies that $\sqrt{v} + \sqrt{p} < \sqrt{v} + \sqrt{v}/2 \leq \tau$.
\end{proof}

\subsection{Decision problems}
Our claims for $\STAT$ from Section \ref{sec:stat-decision} can be adapted to the corresponding values for $\vSTAT$ with only minor adjustments that we explain below.

We define the maximum covered fraction $\dcv\frc$, randomized $\dcv$-cover and statistical dimension with $\dcv$-discrimination analogously to those for $\STAT$. Namely,
\begin{defn}
For a set of distributions $\D$ and a reference distribution $D_0$ over $X$ and $\tau >0$, let $\vrcvr(\D,D_0,\tau)$ denote the smallest $d$ such that there exists a probability measure $\cQ$ over functions from $X$ to $[0,1]$ with the property that for every $D \in \D$,
$$\pr_{\phi \sim \cQ}\lb \left|\sqrt{D[\phi]} - \sqrt{D_0[\phi]} \right| > \tau \rb \geq \fr{d}.$$
\end{defn}
\begin{defn}
For a set of distributions $\D$, a probability measure $\mu$ over $\D$,  a reference distribution $D_0$ over $X$ and $\tau>0$, the maximum covered $\mu$-fraction is defined as \equn{
\dcv\frc(\mu,D_0,\tau) \doteq \max_{\phi:X\rar[0,1]} \left\{\pr_{D\sim \mu}\lb \sdif > \tau\rb \right\}.
}
\end{defn}
\begin{defn}\label{def:sdim-vstat}
  For $\tau>0$, domain $X$ and a decision problem $\B(\D,D_0)$, the \textbf{statistical dimension} with $\dcv$-discrimination $\tau$ of $\B(\D,D_0)$ is defined as $$\RSD_\dcv(\B(\D,D_0),\tau) \doteq \sup_{\mu \in S^\D} (\dcv\frc(\mu,D_0,\tau))^{-1}.$$
\end{defn}

It is easy to see that all results in Section \ref{sec:stat-decision-rand} apply verbatim to the notions defined here (up to replacing the expectation (or an estimate) of every function $\phi$  with its square root and the function range with $[0,1]$ in place of $[-1,1]$). In particular, Lemma \ref{lem:stat-SD-is-rcover} implies that \equn{\RSD_\dcv(\B(\D,D_0),\tau) = \vrcvr(\D,D_0,\tau). } Theorem \ref{thm:random-algorithm2queries} gives the following characterization. We state the bounds for $\vSTAT$ for consistency with the results for $\STAT$. The bounds for $\VSTAT$ are implied by Lemma \ref{lem:vstat-reduction}.

\begin{thm}
\label{thm:vstat-random-algorithm2queries}
Let $\B(\D,D_0)$ be a decision problem, $\tau > 0, \delta \in (0,1/2)$ and $d = \RSD_\dcv(\B(\D,D_0),\tau)$.
Then $$\RQC(\B(\D,D_0),\vSTAT(\tau),1-\delta) \geq  d \cdot (1-2\delta)\mbox{ and}$$
$$\RQC(\B(\D,D_0),\vSTAT(\tau/2),1-\delta) \leq d \cdot \ln(1/\delta).$$
\end{thm}

\subsection{Search problems}
We also define the randomized statistical dimension with $\dcv$-discrimination for search problems analogously.
\begin{defn}\label{def:sdimv-search}
  Let $\Z$ be a search problem over a set of solutions $\F$ and a class of distributions $\D$ over a domain $X$ and let $\tau>0$.
  \remove{
  We define the \textbf{statistical dimension} with $\dcv$-discrimination $\tau$ of $\Z$ as
  $$\SD_\dcv(\Z,\tau) \doteq \sup_{D_0 \in S^X} \inf_{f\in \F} \RSD_\dcv(\B(\D\setminus \Zf,D_0),\tau).$$
  We define the \textbf{randomized statistical dimension} with $\dcv$-discrimination $\tau$ of $\Z$ as
  $$\RSD_\dcv(\Z,\tau) \doteq \sup_{D_0 \in S^X, \ \mu \in S^\D} \left(\max\left\{\mu(\F), \dcv\frc(\mu,D_0,\tau)\right\}\right)^{-1}.$$}
  For success probability parameter $\alpha$, we define the \textbf{randomized statistical dimension} with $\dcv$-discrimination $\tau$ of $\Z$ as
  $$\RSD_\dcv(\Z,\tau,\alpha) \doteq \sup_{D_0 \in S^X}\inf_{\cP \in S^\F} \RSD_\dcv(\B(\D\setminus \Z_\cP(\alpha),D_0),\tau) .$$

\end{defn}

The lower bounds again hold verbatim (up to the same translation). 
\remove{
\begin{thm}
\label{thm:vstat-search-lower}
  Let $X$ be a domain and $\Z$ be a search problem over a set of solutions $\F$
  and a class of distributions $\D$ over $X$. For $\tau > 0$ let $d = \RSD_\dcv(\Z,\tau)$.
  Any SQ that is given access to $\VSTAT(1/(9\tau^2))$ requires at least $d$ calls to the oracle.
\end{thm}
}
\begin{thm}
\label{thm:vstat-search-lower-random}
    For any search problem $\Z$, $\tau >0$ and $\beta > \alpha > 0$,
     $$\RQC(\Z,\vSTAT(\tau),\beta) \geq \RSD_\dcv(\Z,\tau,\alpha) \cdot (\beta - \alpha).$$
\end{thm}

Getting the upper bounds requires a bit more care since the MW updates depend on how well queries distinguish between $D$ and $D_t$. Specifically, instead of condition in eq.~\eqref{eq:mw-updatestrength} we have that
\equn{\left|\sqrt{D_t[\phi_i]} - \sqrt{D[\phi_i]}\right| > \frac{\tau}{3}.}
This implies that
\equ{\left|D_t[\phi_i] - D[\phi_i]\right| \geq \left|\sqrt{D_t[\phi_i]} - \sqrt{D[\phi_i]}\right| \cdot \left(\sqrt{D_t[\phi_i]} + \sqrt{D[\phi_i]}\right) > \frac{\tau^2}{9}\label{eq:mw-updatestrength2}}
and as a result $D_t[\psi_t] - D[\psi_t] \geq \tau^2/9$. We can therefore use the same update but with parameter $\gamma = \tau^2/9$ (instead of $\tau/3$) which leads to a bound of $O(\KLR(\D)/\tau^4)$ on the number of updates. This translates into the following upper bound.

\begin{thm}
\label{thm:vstat-search-upper-random}
For any search problem $\Z$  over a finite class of distributions $\D$ on a finite domain $X$, $\tau >0$ and $\alpha > \delta > 0$,
     $$\RQC(\Z,\vSTAT(\tau/3),\alpha-\delta) = O\lp  \RSD_\dcv(\Z,\tau,\alpha) \cdot \frac{\KLR(\D)}{\tau^4} \cdot\log\lp\frac{ \KLR(\D)}{\tau\delta}\rp\rp .$$
\end{thm}

\begin{remark}
\label{rem:vstat-threshold}
Naturally, the results we give for optimizing and verifiable search can also be extended to $\vSTAT$ in a completely analogous way (and we omit it for brevity). Here the only difference is in how the parameter of the problem needs to be adjusted as a result of using additional queries. For example, for verifiable search our lower bound is for solving $\V_{\theta -\tau}$. Instead, it should be for $\V_{\theta'}$ such that $\sqrt{\theta} - \sqrt{\theta'} = \tau$ or $\theta' = (\sqrt{\theta} - \tau)^2$. Analogous adjustment is needed for the upper bound. This difference can be significant. For example, in the planted $k$-bi-clique problem the verification threshold is $k/n$. If we were using $\STAT$ then it would not be possible to get a meaningful lower bounds for estimation complexity that is below $n^2/k^2$. On the other hand, with $\vSTAT$ we will get a meaningful lower bound with estimation complexity as low as $O(n/k)$.
\end{remark} 
\section{Average discrimination}
\label{sec:average}
In some cases it is analytically more convenient to upper bound the average value by which a query distinguishes between distributions (instead of the fixed minimum). Indeed this has been (implicitly) done in all known lower bounds on SQ complexity. We now show how one can incorporate such averaging into the statistical dimensions that we have defined. The resulting dimensions turn out to be equal, up to a factor of $\tau$, to the corresponding dimension with the fixed minimum discrimination.

The main advantage of this modified dimension is that it allows us to easily relate the dimensions defined in this work to several other notions of dimension that are all closely related to the spectral norm of the discriminating operator. In particular, we show that upper bounds on the $\dc$ norm defined in \cite{FeldmanPV:13} and the average correlation-based dimension in \cite{FeldmanGRVX:12} imply upper bounds on the statistical dimension with the average version of $\dcv$ discrimination.

We conclude this section with a particularly simple dimension that is based solely on average discrimination which we refer to as the combined statistical dimension. It no longer allows to treat the query and estimation complexity separately and only gives a characterization up to a polynomial. Still such dimension suffices for coarse analysis of some problems and we apply it in Sec.~\ref{sec:learning-lines}.

\subsection{Statistical dimension with average discrimination}
\label{sec:statkv}
\remove{
That is instead of
 $$\dcv\frc(\mu,D_0) \doteq \max_{\phi:X\rar[0,1]} \left\{\pr_{D\sim \mu}\lb \sdif > \tau\rb \right\}$$
 }
The average $\dcv$-discrimination is defined as
 $$\dcvi(\mu,D_0) \doteq
 \max_{\phi: X\rar[0,1]} \left\{\E_{D\sim \mu}\left[ \sdif \right]\right\}$$ and we refer to it as {\em $\dcvi$-discrimination}.

For $\D' \subseteq \D$ let $\mu_{|\D'} \doteq \mu(\cdot \cond \D')$. The maximum covered $\mu$-fraction and the randomized statistical dimension for $\dcvi$-discrimination are defined as
\equn{
\dcvi\frc(\mu,D_0,\tau) \doteq \max_{\D' \subseteq \D} \left\{\left. \mu(\D')\ \right|\  \dcvi(\mu_{|\D'},D_0)> \tau \right\}.}
\begin{defn}\label{def:sdim-average}
  For $\tau>0$, domain $X$ and a decision problem $\B(\D,D_0)$, the \textbf{statistical dimension} with $\dcvi$-discrimination $\tau$ of $\B(\D,D_0)$ is defined as $$\RSD_\dcvi(\B(\D,D_0),\tau) \doteq \sup_{\mu \in S^\D} \lp\dcvi\frc(\mu,D_0,\tau)\rp^{-1}.$$
\end{defn}

We will now show that $\RSD_\dcvi(\B(\D,D_0),\tau)$ is closely related to $\RSD_\dcv(\B(\D,D_0),\tau)$.

\begin{lem}
\label{lem:stat-vfrac-to-kvfrac}
For any measure $\mu$ over a set of distributions $\D$, reference distribution $D_0$ and $\tau>0$:
\begin{enumerate}
\item $\dcvi\frc(\mu,D_0,\tau) \geq \dcv\frc(\mu,D_0,\tau)$ and therefore $\RSD_\dcvi(\B(\D,D_0),\tau) \leq \RSD_\dcv(\B(\D,D_0),\tau)$.
\item If $\RSD_\dcv(\B(\D,D_0),\tau)\leq d$ then $\RSD_\dcvi(\B(\D,D_0),\tau/2) \leq \frac{2d}{\tau}$.
\end{enumerate}
\end{lem}
\begin{proof}
For the first direction it is sufficient to observe that if
$$\max_{\phi:X\rar[0,1]} \left\{\pr_{D\sim \mu}\lb \sdif > \tau\rb \right\} = \alpha$$ then for some $\phi$, $\pr_{D\sim \mu}\lb \sdif > \tau]\rb= \alpha$. Defining $\D' \doteq \{D \cond  \sdif > \tau\}$ we get that $\mu(\D') = \alpha$ and $\dcvi(\mu_{|\D'},D_0) > \tau$. Hence $\dcvi\frc(\mu,D_0,\tau) \geq \alpha$.

For the second direction: Let $\mu$ be a measure over $\D$. $\RSD_\dcv(\B(\D,D_0),\tau) \leq d$ implies that $\dcvi\frc(\mu,D_0,\tau) \geq 1/d$. This implies that there exists $\D' \subseteq \D$ such that $\mu(\D') \geq 1/d$ and $\dcvi(\mu_{|\D'},D_0) > \tau$.

By the definition of $\dcvi(\D',D_0)$, we know that there exists a function $\phi:X\rar[0,1]$ such that
 \equ{\E_{D\sim \mu_{|\D'}}\lb \sdif \rb> \tau . \label{eq:good_h}}
Let $$\D_\phi \doteq \left\{ D \in \D'\ \left|\ \sdif > \tau/2 \right. \right\} .$$

By observing that eq.~\eqref{eq:good_h} implies,
\equ{\tau < \E_{D \sim \mu_{|\D'}}\lb \sdif \rb \leq \mu(\D_\phi \cond \D') + \frac{\tau}{2} \cdot (1-\mu(\D_\phi \cond \D')),\label{eq:norm2fraction}}
we get that  $\mu(\D_\phi \cond \D') \geq \tau/2$. This means that $\mu(\D_\phi) =  \mu(\D_\phi \cond \D') \cdot \mu(\D') \geq \tau/2 \cdot 1/d$. This hold for every $\mu$ and therefore $\RSD_\dcvi(\B(\D,D_0),\tau/2) \leq \frac{2d}{\tau}$.
\end{proof}

As an immediate Corollary of Lemma \ref{lem:stat-vfrac-to-kvfrac} and Theorem \ref{thm:vstat-random-algorithm2queries} we get the following characterization.
\begin{cor}
\label{cor:vstat-decision-kv}
Let $\B(\D,D_0)$ be a decision problem, $\tau > 0, \delta \in (0,1/2)$ and $d= \RSD_\dcvi(\B(\D,D_0),\tau)$.
Then $$\RQC(\B(\D,D_0),\vSTAT(\tau),1-\delta) \geq  d  (1-2\delta)\mbox{ and}$$
$$\RQC(\B(\D,D_0),\vSTAT(\tau/2),1-\delta) \leq d \cdot \frac{2\ln(1/\delta)}{\tau}.$$
\end{cor}

The same relationship holds for the corresponding dimension of search problems. This follows immediately from the fact that all dimensions that we have defined rely on $\RSD_\dcv$. Hence we can apply Lemma \ref{lem:stat-vfrac-to-kvfrac} to obtain lower bounds based on $\SD_\dcvi$ and $\RSD_\dcvi$ which are identical to those based on $\SD_\dcv$  and $\RSD_\dcv$. The corresponding upper bounds have an additional factor of $1/\tau$ in the query complexity. We omit the repetitive statements.

We also analogously define an average case version $\dci$-discrimination as
$$\dcii(\mu,D_0) \doteq
 \max_{\phi: X\rar[-1,1]} \left\{\E_{D\sim \mu}\left[ \dif \right]\right\}$$
 and use it as the basis to define $\dcii\frc$ and $\RSD_\dcii$. The resulting dimensions are equivalent to $\RSD_\dci$ up to a factor of $1/\tau$ in the same way.

\subsection{Relationships between norms}
\label{sec:norms}
We first confirm that the norms we have defined preserve the relationship between the oracles $\vSTAT$ and $\STAT$:
$$\dcvi(\mu,D_0)\geq \frac{1}{4} \cdot \dcii(\mu,D_0) \geq \frac{1}{2}\cdot\dcvi(\mu,D_0)^2. $$
(see Lemma \ref{lem:k1k2} for a proof).
This implies that for any search or decision problem $\Z$, 
$\RSD_\dcv(\Z,\tau) \leq \RSD_\dci(\Z,\tau/4) \leq \RSD_\dcv(\Z,\tau^2/2)$.

We now consider the $\dc$-norm defined in \cite{FeldmanPV:13} as follows:
\equn{
\dc(\D,D_0) \doteq \frac{1}{|\D|} \cdot \max_{\phi :X \rar \R,\ \|\phi\|_{D_0} = 1} \left\{ \sum_{D \in \D} \dif\right\} .
}
where the (semi-)norm of $\phi$ over $D_0$ is defined as $\|\phi\|_{D_0} = \sqrt{D_0[\phi^2]}$.

The statistical dimension with $\dc$ norm for decision problems is defined as (\cite{FeldmanPV:13})
$$\SD_\dc(\B(\D,D_0),\tau) \doteq \sup_{\D_0 \subseteq \D,0< |\D_0|<\infty} \lp\dc\frc(\D_0,D_0,\tau)\rp^{-1},$$
where $$\dc\frc(\D_0,D_0,\tau) \doteq \max_{\D' \subseteq \D} \left\{\left.\frac{|\D'|}{|\D_0|}\ \right|\  \dc(\D',D_0)> \tau \right\}.$$

Note that this is exactly the $\dc$ version of the deterministic statistical dimension we defined in Section \ref{sec:stat-decision-det}.
We now show that $\dc$ leads to a smaller dimension than $\dcvi$. For convenience we extend the definition of $\dc$ to measures over $\D$ in the straightforward way. 
\begin{lem}
\label{lem:k2-kv}
For any measure $\mu$ over a set of distributions $\D$ and a reference distribution $D_0$ over $X$,
$$\dcvi(\mu,D_0) \leq \dc(\mu,D_0).$$
\end{lem}
\begin{proof}
\alequn{
\dcvi(\mu,D_0) &\equiv \max_{\phi :X \rar [0,1]} \E_{D\sim \mu} \lb \left| \sqrt{ D[\phi]}-\sqrt{D_0[\phi]}\right|\rb  \nonumber \\
&= \max_{\phi :X \rar [0,1]} \E_{D\sim \mu} \lb \frac{\dif }{\sqrt{ D[\phi]}+\sqrt{D_0[\phi]}} \rb   \\
&\leq \max_{\phi :X \rar [0,1]} \E_{D\sim \mu} \lb \frac{\dif }{\sqrt{ D[\phi^2]}+\sqrt{D_0[\phi^2]}} \rb \\
&\leq \max_{\phi :X \rar [0,1]} \E_{D\sim \mu} \lb \frac{\dif }{\|\phi\|_{D_0}} \rb  \nonumber\\
&\leq \max_{\phi :X \rar \R,\ \|\phi\|_{D_0} = 1}  \E_{D\sim \mu} \lb \dif\rb \equiv \dc(\D,D_0) .}
\end{proof}
 This immediately implies the following corollary.
\begin{cor}
Let $\B(\D,D_0)$ be a decision problem over a class of distributions $\D$ over a domain $X$ and reference distribution $D_0$ and $\tau>0$. Then
$$\RSD_\dcvi(\B(\D,D_0),\tau) \geq \RSD_\dc(\B(\D,D_0),\tau) \geq \SD_\dc(\B(\D,D_0),\tau).$$
\end{cor}
This implies that lower bounds on $\SD_2$ (such as those proved in \cite{FeldmanPV:13,FeldmanGV:15}) directly imply lower bounds on $\RSD_\dcv$ defined here.

For completeness, we also describe the relationship to a simpler notion of average correlation introduced in \cite{FeldmanGRVX:12}.  Specifically, assuming that for $D \in \D$, every $x$ that is in the support of $D$ is also in the support of $D_0(x)$ we can define a function $\hat{D}(x) \doteq \frac{D(x)}{D_0(x)}-1$. We can then define the average correlation as
$$\rho(\D,D_0) = \fr{|\D|^2} \sum_{D,D'\in \D} \left|D_0\lb\hat{D} \cdot \hat{D'} \rb\right| .$$ Note that when $D = D'$, the quantity $D_0[\hat{D}^2]$ is known as the $\chi^2(D, D_0)$ divergence (or distance).
 Using this notion, \cite{FeldmanGRVX:12} defined the statistical dimension with average correlation $\gamma$:
$$\SD_\rho(\B(\D,D_0),\gamma) \doteq \sup_{\D_0 \subseteq \D,0< |\D_0|<\infty} \lp\rho\frc(\D_0,D_0,\gamma)\rp^{-1},$$
where $$\rho\frc(\D_0,D_0,\gamma) \doteq \max_{\D' \subseteq \D} \left\{\left.\frac{|\D'|}{|\D_0|}\ \right|\  \rho(\D',D_0)> \gamma \right\}.$$
Then it is not hard to prove that $\rho(\D,D_0) \geq (\dc(\D,D_0))^2$ and therefore $\SD_\dc(\B(\D,D_0),\tau) \geq \SD_\rho(\B(\D,D_0),\tau^2)$ (see Lemma \ref{lem:k2-rho} for proof).

As it was shown in \cite{FeldmanGRVX:12}, the statistical dimension with average correlation is a generalization of a statistical dimension notions in the learning theory that are based on pairwise correlations (introduced in \cite{BlumFJ+:94}).
\paragraph{Relationship to matrix norms:}
The $\dc$-discrimination norm is closely related to the (weighted) spectral norm of the discriminating operator defined as
\equn{
\dcsp(\mu,D_0) \doteq  \max_{\phi :X \rar \R,\ \|\phi\|_{D_0} = 1}\lp \E_{D \sim \mu} \left( D[\phi]-D_0[\phi]\right)^2\rp^{1/2} .
}
Clearly, $\dc(\mu,D_0) \leq \dcsp(\mu,D_0)$ and hence upper bounds on the spectral norm can also be used to get upper bounds on $\dcvi$-norm. This means that $\dcsp(\mu,D_0)$ can be used in place of $\dcvi$ (and $\dcii$) in any of our lower bounds.

Note that $\dcsp$ can be seen as a weighted spectral norm of the matrix $A$ whose rows are indexed by $D\in \D$, the columns are indexed by $x \in X$ and $A[D,x] = D(x) - D_0(x)$. Then, using $\|w\|_{\mu}$ to denote $\sqrt{\E_{D\sim \mu} w_D^2}$, we have
$$\dcsp(\mu,D_0) \equiv \max_{\|\phi\|_{D_0} = 1} \|A\phi\|_{\mu} = \left\| B_\mu^{1/2} \cdot A \cdot B_{D_0}^{-1/2}\right\|_2 ,$$
Where $B_\mu$ is the diagonal $|\D|\times|\D|$ matrix such that $B[D,D] = \mu(D)$ and, similarly,
$B_{D_0}$ is the $|X|\times |X|$ matrix such that $B_{D_0}[x,x] = D_0(x)$. From this point of view, $$\dc(\mu,D_0) \equiv \left\| B_\mu \cdot A \cdot B_{D_0}^{-1/2}\right\|_{2\rar 1} \mbox{ and}$$
$$\dcii(\mu,D_0) \equiv \left\| B_\mu \cdot A \right\|_{\infty \rar 1}.$$

\subsection{Combined statistical dimension}
\label{sec:combined}
For some problems we are interested in a coarser picture in which it is sufficient to estimate the maximum of the query complexity and the estimation complexity up to a polynomial.
For such cases we can avoid our fractional notions and get a simpler {\em combined statistical dimension} based on average discrimination.
In such settings the distinction between $\STAT$ and $\VSTAT$ is usually not essential and therefore we use $\dcii$ for simplicity and state the results for $\STAT$. Analogous results hold for $\vSTAT$ when one uses $\dcvi$ instead and we describe the small differences in Remark \ref{rem:combined-for-vstat}. 
\begin{defn}
   For a decision problem $\B(\D,D_0)$, the \textbf{combined  statistical dimension} of $\B(\D,D_0)$ with $\dcii$-discrimination is defined as $$\CRSD_\dcii(\B(\D,D_0)) \doteq \sup_{\mu \in S^\D} \lp \dcii(\mu,D_0)\rp^{-1}.$$
\end{defn}
To show that the combined dimension characterizes the randomized SQ complexity (up to a polynomial) we demonstrate that it can be related to $\RSD_\dci$.
\begin{lem}
\label{lem:combined2disjoint-decision}
For any decision problem $\B(\D,D_0)$, if $\CRSD_\dcii(\B(\D,D_0))=d$ then
$\RSD_\dci(\B(\D,D_0),1/(3d)) \leq 3d$ and
 for every $\tau>0$, $\RSD_\dci(\B(\D,D_0),\tau) > d\tau$.
\end{lem}
\begin{proof}
Both $\RSD_\dci$ and $\CRSD_\dcii$ have a supremum over $\mu \in S^\D$ and therefore it is sufficient to prove that for a fixed $\mu$ such that $\fr{\dcii(\mu,D_0)} = d$, we have that $\dcii\frc(\mu,D_0,1/(3d)) \leq 3d$ and $\dci\frc(\mu,D_0,\tau) > d\tau$.

For the first part:  There exists a function $\phi:X\rar \pmr$ such that
$\E_{D \sim \mu} \lb \dif \rb > 1/d$.
Define $$\D' \doteq \left\{ D \in \D\ \left|\ \dif > 1/(3d)  \right. \right\}.$$
Now, from \equn{\fr{d} \leq \E_{D \sim \mu} \lb \dif \rb \leq 2 \cdot \pr_{D \sim \mu}\lb  \dif > \fr{3d} \rb + \frac{1}{3d} \cdot 1,}
we get that $\pr_{D \sim \mu}\lb  \dif > \fr{3d} \rb \geq 1/(3d)$. This implies that $\dci\frc(\mu,D_0,1/(3d)) \leq 3d$.

For the second part:  For every function $\phi:X\rar\pmr$,
$$\E_{D \sim \mu}\lb \dif \rb > \pr_{D \sim \mu}\lb \dif > \tau\rb \cdot \tau .$$
Therefore, $$ \max_{\phi:X\rar\pmr} \pr_{D \sim \mu}\lb \dif > \tau \rb < \frac{\dcii(\mu,D_0)}{\tau} \leq  \fr{d\tau}.$$
This means that $\dci\frc(\mu,D_0,\tau) > d\tau$.
\end{proof}

Plugging Lemma  \ref{lem:combined2disjoint-decision} into Theorems \ref{thm:random-algorithm2queries} gives the following bounds.
\begin{thm}
\label{thm:combined-decision}
Let $\B(\D,D_0)$ be a decision problem, $\tau > 0, \delta \in (0,1/2)$  and let $d = \CRSD_\dcii(\B(\D,D_0))$.
Then $$\RQC(\B(\D,D_0),\STAT(\tau),1-\delta) \geq d \cdot \tau \cdot (1-2\delta)\mbox{ and}$$
$$\RQC(\B(\D,D_0),\STAT(1/(3d)),1-\delta) \leq 3 \cdot d \cdot \ln(1/\delta).$$
\end{thm}

Lem.~\ref{lem:combined2disjoint-decision} implies that combined dimension can be extended to SQ complexity of search problems in a straightforward way.

\begin{defn}\label{def:csdim-search}
 For a search problem $\Z$ and $\alpha>0$, the \textbf{combined randomized statistical dimension} of $\Z$ as  $$\CRSD_\dcii(\Z,\alpha) \doteq \sup_{D_0 \in S^X}\inf_{\cP \in S^\F} \CRSD_\dcii(\B(\D\setminus \Z_\cP(\alpha),D_0)) .$$
\end{defn}

Plugging Lemma  \ref{lem:combined2disjoint-decision} into Theorems \ref{thm:stat-search-lower-random} and \ref{thm:stat-search-upper-random} gives the following characterization.
\begin{thm}
\label{thm:comb-stat-search-upperlower}
  Let $\Z$ be a search problem, $\beta> \alpha>0, \tau >0$ and let $d = \CRSD_\dcii(\Z,\alpha)$. Then $\RQC(\Z,\STAT(\tau),\beta) \geq d(\beta-\alpha)\tau$ and for every $\delta > 0$,
    $$\RQC(\Z,\STAT(1/(3d)),\alpha-\delta) =\tilde{O}(d^3 \cdot \KLR(\D) \cdot \log(1/\delta)) .$$
\end{thm}

Analogous versions of the characterization for verifiable and optimizing search can be easily obtained. In Section \ref{sec:learning} we describe the combined dimension of verifiable search in the context of PAC learning.

\begin{remark}
\label{rem:combined-for-vstat}
It is easy to see that an analogous characterization can be established for $\vSTAT$ using $\dcvi$ in place of $\dcii$. The only differences would be: constant $2$ instead of $3$ in Lemma \ref{lem:combined2disjoint-decision} and its implications (since the range of functions is $[0,1]$) and $d^5$ in place of $d^3$ in Thm.~\ref{thm:comb-stat-search-upperlower} which would be based on Thm.~\ref{thm:vstat-search-upper-random}.
\end{remark}

\remove{
\begin{thm}
\label{thm:comb-stat-search-lower-random}
  Let $\Z$ be a search problem and let $d = \CRSD_\dci(\Z)$.
  There exists a measure $\mu$ over $\D$ such that any SQ algorithm that solves $\Z$ with probability at least $1-\delta$ over the choice of the input distribution $D$ from $\mu$ and randomness of the algorithm) and is given access to $\STAT(\tau)$ requires at least $(1-\delta)d\tau-1$ queries.
\end{thm}

} 
\section{Applications to PAC learning}
\label{sec:learning}
We now instantiate our dimension in the PAC learning setting and provide some example applications.

\subsection{Characterization of the SQ Complexity of PAC Learning}
\label{sec:learning-dim}
Let $\C$ be a set of Boolean functions over some domain $Z$. We recall that in PAC learning the set of input distributions $\D_\C = \{P^f \cond P\in S^{Z},\ f\in \C\}$, where $P^f$ denotes the probability distribution on the examples $(z,f(z))$ where $z\sim P$. The set of solutions is all Boolean functions over $Z$ and for an input distribution $P^f$ and $\eps >0$ the set of valid solutions are those functions $h$ for which $\pr_{(z,b)\sim P^f} [h(z) \neq b ]\leq \eps$. This implies that PAC learning is a verifiable search problem with parameter $\eps$. Now the set of distributions that cannot lead to valid solutions is exactly the set of distributions that cannot be predicted with error lower than $\eps$.
More formally, for a distribution $D_0$ over $Z \times \pmi$ we denote by $\err(D_0)$ the Bayes error rate of $D_0$, that is $$\err(D_0) = \sum_{z \in Z} \min\{D_0(z,1), D_0(z,-1)\} = \min_{h:Z\rar \pmi} \pr_{(z,b) \sim D_0} [h(z) \neq b].$$
We denote the problem of PAC learning $\C$ to accuracy $\eps$ by $\Learn_{\mbox{PAC}}(\C,\eps)$.
Using Def.~\ref{def:sdim-verifiable} we get the following notion:
\begin{defn}\label{def:sdim-learning}
  For a concept class $\C$ over a domain $Z$ and $\eps,\tau > 0$,
  $$\RSDV_\dcv(\Learn_{\mbox{PAC}}(\C,\eps),\tau) \doteq \sup_{D_0 \in S^{Z\times\pmi},\ \err(D_0) >\eps} \RSD_\dcv(\B(\D_\C,D_0),\tau) .$$
\end{defn}
In this case $\KLR(\D_\C) \leq \ln(2|Z|)$. Therefore we get the following upper and lower bounds:
\begin{thm}
For a concept class $\C$  over a domain $Z$ and $\eps,\tau,\beta,\delta > 0$, let $d= \RSDV_\dcv(\Learn_{\mbox{PAC}}(\C,\eps),\tau)$.
Then $$\RQC(\Learn_{\mbox{PAC}}(\C,\eps -  2\tau\sqrt{\eps}), \vSTAT(\tau),\beta) \geq \beta d -1 \mbox{ and}$$ $$\RQC(\Learn_{\mbox{PAC}}(\C,\eps+ 2\tau\sqrt{\eps}),\vSTAT(\tau/3),1-\delta) =  \tilde{O}(d \cdot \log(|Z|) \cdot \log(1/\delta)/\tau^2).$$
\end{thm}
We remark that the term $2\tau\sqrt{\eps}$ comes from the condition on the threshold $|\sqrt{\eps'}-\sqrt{\eps}| \leq \tau$ that results from the adaptation of the lower bound for verifiable search to $\vSTAT$ discussed in Remark \ref{rem:vstat-threshold}. Note that this leads to meaningful bounds only when $\tau < \sqrt{\eps/4}$ or, equivalently, estimation complexity being $\Omega(1/\eps)$. It is not hard to show that this condition is necessary since the query complexity of PAC learning non-trivial classes of functions with $o(1/\eps)$ estimation complexity is infinite.

We can similarly characterize the SQ complexity of distribution-specific PAC learning. For a distribution $P$ over $Z$ we denote the set of all input distributions by $\D_{\C,P} \doteq \{ P^f \cond f\in \C\}$, the set of all distributions over $Z\times \pmi$ whose marginal is $P$ by $\D_{P}$ and the learning problem by $\Learn_{\mbox{PAC}}(\C,P,\eps)$.
\begin{defn}\label{def:sdim-learning-dis-specific}
  For a concept class $\C$, distribution $P$ over a domain $Z$ and $\eps,\tau > 0$,
  $$\RSDV_\dcv(\Learn_{\mbox{PAC}}(\C,P,\eps),\tau) \doteq \sup_{D_0 \in \D_P,\ \err(D_0) >\eps} \RSD_\dcv(\B(\D_{\C,P},D_0),\tau) .$$
\end{defn}

Then observing that $\KLR(\D_{\C,P}) \leq \ln(2)$ we obtain the following tight characterization of distribution-specific learning.
\begin{thm}
 For a concept class $\C$, distribution $P$, $\eps,\tau,\delta,\beta > 0$, let $d= \RSDV_\dcv(\Learn_{\mbox{PAC}}(\C,P,\eps),\tau)$. Then $$\RQC(\Learn_{\mbox{PAC}}(\C,P,\eps -  2\tau\sqrt{\eps}), \vSTAT(\tau),\beta) \geq \beta d -1 \mbox{ and}$$ $$\RQC(\Learn_{\mbox{PAC}}(\C,P,\eps+ 2\tau\sqrt{\eps}),\vSTAT(\tau/3),1-\delta) =  \tilde{O}(d \log(1/\delta)/\tau^2).$$
\end{thm}

This characterization of distribution-specific learning can be seen as a strengthening of the characterization in \cite{Feldman:12jcss}. There 
a dimension based on pairwise correlations was described that only characterizes the estimation complexity up to polynomial factors.

The statistical dimensions introduced above are particularly suitable for the study of {\em attribute-efficient} learning, that is learning in which the number of samples (or estimation complexity) is much lower than the running time (query complexity). Several basic questions about the SQ complexity of this class of problems are still unsolved \cite{Feldman14:open}.

Naturally, the combined statistical dimension can also be used in this setting.
\begin{defn}\label{def:sdim-learning-combined}
  For a concept class $\C$  over domain $Z$ and $\eps > 0$,
  $$\CRSD_\dcii(\Learn_{\mbox{PAC}}(\C,\eps)) \doteq \sup_{D_0 \in S^{Z\times\pmi},\ \err(D_0) >\eps} \CRSD_\dcii(\B(\D_\C,D_0).$$
\end{defn}
The lower and the upper bounds follow immediately from Thms.~\ref{thm:stat-verifiable-upper} and \ref{thm:stat-verifiable-upper} together with Lemma \ref{lem:combined2disjoint-decision}.

\begin{thm}
\label{thm:learn-combined}
 For a concept class $\C$ over domain $Z$ and $\eps,\delta,\beta > 0$, let $d= \CRSD_\dcii(\Learn_{\mbox{PAC}}(\C,\eps))$. Then $$\RQC(\Learn_{\mbox{PAC}}(\C,\eps - 1/\sqrt{d}), \STAT(1/\sqrt{d}),\beta) \geq \beta\sqrt{d} -1\mbox{ and}$$ $$\RQC(\Learn_{\mbox{PAC}}(\C,P,\eps+ 3/d),\STAT(1/(9d)),1-\delta) =  \tilde{O}(d^3 \cdot \log(|Z|) \cdot \log(1/\delta)).$$
\end{thm}

Agnostic learning can be characterized by adapting analogously the characterization for optimizing search problems.

\subsection{Distribution-independent vs distribution-specific SQ learning}
\label{sec:learning-lines}

We now give a simple application of our lower bound to demonstrate that for SQ PAC learning distribution-specific complexity cannot upper-bound the distribution-independent SQ complexity. Further the gap is exponential even if one is allowed a dependence on the input point size $\log(|X|)$ (otherwise, the class of thresholds functions on a discretized interval can be used to prove this separation using a simple description-length-based argument).
 This is in contrast to the sample complexity of learning since for every concept class $\C$, there exists a distribution $P$ such that the sample complexity of PAC learning $\C$ over $P$ with error $1/4$ is $\Omega(\VCDIM(\C))$ (\eg \cite{Shalev-ShwartzBen-David:2014}). Our lower bound also implies that the hybrid SQ model in which the learner has access to unlabeled samples from the marginal distribution $P$ in addition to SQs is strictly stronger than the ``pure" SQ model. As was observed in \cite{FeldmanKanade:12}, the SQ complexity of distribution-independent learning in the hybrid model is exactly the maximum over all distribution $P$ of the SQ complexity of (distribution-specific) learning over $P$.

The key to this separation is a lower bound for the class of linear functions over a finite field of large characteristic. Specifically,
for $a=(a_1,a_2)\in \GF_p^2$, we define a line function $\ell_a$ over $\GF_p^2$ as $\ell_{a}(z) = 1$ if and only if $a_1z_1 + a_2 = z_2 \mod p$. Let $\Line_p \doteq \{\ell_a \cond a \in \GF_p^2\}$.
 We now prove that any SQ algorithm for (distribution-independent) PAC learning of $\Line_p$ with $\eps = 1/2-c \cdot p^{-1/4}$ (for some constant $c$), must have complexity of $\Omega(p^{1/4})$.
We can now prove our lower bound.
\begin{thm}
\label{thm:line-lower-bound}
For any prime $p$, any randomized algorithm that is given access to $\STAT(1/t)$ and PAC learns $\Line_p$ with error $\eps < 1/2 - 1/t$ and success probability at least $2/3$ requires at least $t/2-1$ queries, where $t = (p/32)^{1/4}$.
\end{thm}
\begin{proof}
Let $\D = \D_{\Line_p}$. We will lower bound $\CRSD_\dcii(\Learn_{\mbox{PAC}}(\Line_p,\eps))$ defined in Def.~\ref{def:sdim-learning-combined}.  Let $D_0$ be the uniform distribution over $X \doteq \GF_p^2 \times \pmi$. Note that $\err(D_0)=1/2$. We now show that $\CRSD_\dcii(\B(\D,D_0)) \geq \sqrt{p/32}$.
By definition, $$\CRSD_\dcii(\B(\D,D_0)) = \sup_{\mu \in S^\D} \lp\dcii(\mu,D_0)\rp^{-1}.$$
We now choose $\mu$. For $a \in \GF_p^2$ we define $P_a$ to be the distribution over $\GF_p^2$ that has density $1/(2p^2)$ on all $p^2-p$ points where $\ell_a = -1$ and has density $1/(2p)+1/(2p^2)$ on all the $p$ points where $\ell_a = 1$. 
We then define $D_a \doteq P_a^{\ell_a}$, namely the distribution over examples of $\ell_a$, whose marginal over $\GF_p^2$ is $P_a$. Let $\D' \doteq \{D_a \cond a \in \GF_p^2\}$ and let $\mu$ be the uniform distribution over $\D'$.

Now to estimate $\dcii(\mu,D_0)$ we use that by Lemmas \ref{lem:k1k2}, \ref{lem:k2-kv} and \ref{lem:k2-rho}, $$\dcii(\mu,D_0) \leq 4 \dcvi(\mu,D_0) \leq 4\dc(\mu,D_0) \equiv 4\dc(\D',D_0) \leq 4\sqrt{\rho(\D',D_0)}.$$

So it suffices to upper bound $\rho(\D',D_0)$. An alternative would be to use the spectral norm $\dcsp(\mu,D_0)$ which would give the a slightly weaker bound (and require the same analysis).

To calculate the average correlation $\rho(\D',D_0)$ we first note that
$$\hat{D}_a(z,b) = \left\{\arr{ll}{ p & \mbox{if } b=1,\ \ell_a(z)=b \\
                                    0 & \mbox{if } b=-1,\ \ell_a(z)=b \\
                                    -1 & \mbox{if } \ell_a(z)\neq b } \right.$$
Now for $a,a' \in \GF_p^2$ the correlation is:
$$\left|D_0\lb \hat{D}_a \cdot \hat{D}_{a'} \rb\right| \leq \left\{\arr{ll}{ p/2 + 1 & \mbox{if } a=a' \\
                                    1 & \mbox{if (parallel) } a_1 = a'_1 \mbox{ and } a_2 \neq a'_2 \\
                                    1/p^2 & \mbox{otherwise }  } \right.$$
Therefore $$\rho(\D',D_0) \leq \fr{p^4} \cdot \lp p^2\lp \frac{p}{2}+1 \rp + p^2(p-1) + p^4\fr{p^2}\rp \leq \frac{2}{p},$$ and thus $\dcii(\mu,D_0) \leq 4\sqrt{2/p}$.
Applying Theorem \ref{thm:learn-combined} we get that any randomized SQ algorithm with access $\STAT((32/p)^{1/4})$ that PAC learns $\Line_p$ with error lower than $\eps = 1/2 - (32/p)^{1/4}$ and success probability at least $1/2$ requires $(p/32)^{1/4}/2-1$ queries.
\end{proof}

Next, we show that $\Line_p$ is PAC learnable for any fixed distribution $P$ over $\GF_p^2$.
\begin{thm}
\label{thm:line-upper}
Let $P$ be a distribution over $\GF_p^2$. There exists an (efficient) algorithm that PAC learns $\Line_p$ over $P$ with error $\eps$ using $O(1/\eps^2)$ queries $\STAT(\eps^2/13)$.
\end{thm}
\begin{proof}
We first find a hypothesis that predicts correctly on all the heavy points, that is points whose weight is at least $\eps^2/12$. Let $W = \{z \cond P(z) \geq \eps^2/12\}$. Clearly $|W| \leq 12/\eps^2$. For each $z\in W$ we can find the value of the target function $f$ on $z$ by asking a query to $\STAT(\eps^2/13)$. Let $h_W$ be the function that equals to the target function in the set $W$ and is $-1$ everywhere else.

We measure the error of hypothesis $h_W$ with accuracy $\eps/6$. If the error is less than $5\eps/6$ then we are done. Otherwise, only positive points of the target function outside of $W$ contribute to the error of $h_W$ and therefore we know that the weight of these points is at least $2\eps/3$. They must all lie on the same line. Now we claim that there can be at most $2/\eps$ lines such that probability of their positive points outside $W$ is at least $2\eps/3$. This is true since if there are $2/\eps$ such lines: then each of those lines shares at most $2/\eps-1$ points with all other lines and therefore has at least $2\eps/3 - (2/\eps-1) \cdot \eps^2/12 > \eps/2$ unique weight in its positive points outside $W$. This means that the total weight in $2/\eps$ lines is more than 1.
We know $P$ and therefore we can find the target function among those ``heavy" lines by measuring its error using a query to $\STAT(\eps/2)$.
\end{proof}

\subsection{Learning of $\Line_p$ with noise}
\label{sec:line-noise}
We also demonstrate that our lower bound for $\Line_p$ implies a complete separation between learning with noise and (distribution-independent) SQ learning.
When learning with random classification noise of rate $\eta$, the learner observes examples of the form $(z,f(z) \cdot b)$ where $b =1$ with probability $1-\eta$ and $b=-1$ with probability $\eta$ (and $\E[b] = 1-2\eta$) \cite{AngluinLaird:88}. An efficient learning algorithm in this model needs to find a hypothesis with error $\eps$ (on noiseless examples) in time polynomial in $1/\eps$, $1/(1-2\eta)$, $\log(|\C|)$ and $\log(|X|)$. \citet{Kearns:98} has famously showed that any $\C$ that can be learned efficiently using SQs can also be learned efficiently with random classification noise. He also asked whether efficient SQ learning is equivalent to efficient learning with noise.

This question was addressed by \cite{BlumKW:03} whose influential work demonstrated that there exists a class of functions that is learnable efficiently with random classification noise for any constant $\eta <1/2$ but requires super-polynomial time for SQ algorithms. More specifically, the class consists of parity functions on first $\log n \cdot \log\log n$ out of $n$ Boolean variables, it is learnable from noisy examples in $(1-2\eta)^{O(\log n)}$ time and the SQ complexity of learning this class is $n^{\Omega(\log\log n)}$. Note that this result does not fully answer the question in \cite{Kearns:98} since the separation disappears when $1-2\eta = 1/\log n$ whereas SQ algorithms would give a polynomial in $n$ algorithm for $1-2\eta = 1/\poly(n)$. It is also relatively weak quantitatively.

Our lower bound for $\Line_p$ implies strong separation for distribution independent SQ learning, making progress in understanding of this open problem. The upper bound follows easily from the fact that the VC-dimension of $\Line_p$ is 2 and we describe it here briefly for completeness. Indeed, it is easy to see that $\Line_p$ can be learned in the agnostic model with excess error $\eps$ and success probability $2/3$ in time $O(\log(p)/\eps^{6})$. All one needs is to get $O(1/\eps^2)$ random examples, try all the line functions that pass through a pair of examples  and pick the one that agrees best with the labels of all the examples. Standard uniform convergence results for agnostic learning (\eg \cite{Shalev-ShwartzBen-David:2014}) imply that this algorithm will have excess error of at most $\eps$ with probability at least $2/3$. Agnostic learning with excess error of $(1-2\eta)\eps$ implies PAC learning with error $\eps$ and random classification noise of rate $\eta$ \cite{Kearns:98}. Therefore we obtain an exponential separation with polynomial dependence on $1/(1-2\eta)$:
\begin{fact}
For any prime $p$ and $\eta \neq 1/2$, there exists an algorithm that PAC learns $\Line_p$ using $O(1/(\eps(1-2\eta))^2)$ examples corrupted by random classification noise of rate $\eta$ and $O(\log(p)/(\eps(1-2\eta))^6)$ time.
\end{fact}
We note that the open problem remains not fully resolved for distribution-specific SQ learning or, equivalently, the hybrid SQ model. The lower bound in \cite{BlumKW:03} applies to this stronger model.

\section{Conclusions}
Given the central role that the SQ model plays in learning theory, private data analysis and several additional applications, techniques for understanding the SQ complexity can shed light on the complexity of many important theoretical and practical problems. As we demonstrate here, the SQ complexity of any problems defined over distributions can be fairly tightly characterized by relatively simple (compared to other general notions of complexity) parameters of the problem. We believe that this situation is surprising and merits further investigation: SQ algorithms capture most approaches used for statistical problems yet proper understanding of the computational complexity itself is still well outside of our reach. Understanding of the significance of our characterization in the context of specific problems is an interesting avenue for further research.


While we have described several techniques for simplifying the analysis of our statistical dimensions, a lot more work remains in adapting and simplifying the dimensions to specific types of problems (\eg convex optimization or Boolean constraint satisfaction). In particular, it is interesting to understand for which problems one can avoid the $\KLR(\D)/\tau^2$ overhead of our characterization. We also have relatively few analysis techniques for the norms of operators that emerge in the process. Finally, the SQ complexity of many concrete problems is still unknown (\eg \cite{Sherstov:08,Feldman14:open}).

\section*{Acknowledgements}
I thank Sasha Sherstov and Santosh Vempala for many insightful discussions related to this work. I am especially grateful to Justin Thaler for the discussions that stimulated the work on the results in Section \ref{sec:learning-lines}.



\printbibliography

\appendix
\section{Examples of problems over distributions}
\label{sec:problem-examples}
\paragraph{Supervised learning:}
In PAC learning \cite{Valiant:84}, for some set $Z$ and a set of Boolean functions $\C$ over $Z$, we are given access to randomly chosen examples $(z,f(z))$ for some unknown $f\in \C$ and $z$ chosen randomly according to some unknown distribution $P$ over $Z$. The learning algorithm is given an error parameter $\eps>0$ and its goal is to find a function $h:Z \rar \pmi$ such that $\pr_{z\sim P}[f(z) \neq h(z)] \leq \eps$. In other words, the domain is $X = Z \times \pmi$ and the set of input distributions $\D_\C = \{P^f \cond P\in S^{Z},\ f\in \C\}$, where $P^f$ denotes the probability distribution such that for every $z \in Z$, $P^f(z,f(z)) = P(z)$ and $P^f(z,-f(z)) = 0$. The set of solutions is all Boolean functions over $Z$ and for an input distribution $P^f$ and $\eps >0$ the set of valid solutions are those functions $h$ for which $\pr_{(z,b)\sim P^f} [h(z) \neq b ]\leq \eps$. Usually we are interested in efficient learning algorithms in which case the running time of the algorithm and the time to evaluate $h$ should be polynomial in $1/\eps$, $\log(|X|)$ and $\log(|\C|)$.

In agnostic PAC learning \cite{KearnsSS:94}, the set of input distributions is $S^X$ and the goal is to find a function $h$ such that
$$\pr_{(z,b)\sim D} [h(z) \neq b ]\leq \eps + \min_{f\in \C}\left\{\pr_{(z,b)\sim D} [f(z) \neq b ]\right\} ,$$ where $\eps$ is referred to as excess error.
In distribution-specific (agnostic) PAC learning the marginal distribution over $Z$ is fixed to some $P$.

Agnostic PAC learning is a special case of more general supervised learning setting \cite{Vapnik:98} in which instead of Boolean functions we have functions with some range $Y$ and there is a loss function $L:Y\times Y$ that we want to minimize. In other words, the set of input distributions is a subset of all distributions over $Z \times Y$ and the goal is to output a function $h:Z\rar Y$ such that
$$\E_{(z,y)\sim D} [L(y,h(z))] \leq \eps + \min_{f\in \C}\left\{\E_{(z,y)\sim D} [L(y,f(z)]\right\} .$$
\paragraph{Random constraint satisfaction:}
Closely related to learning are random constraint satisfaction problems. Here the domain $X$ is the set of some Boolean predicates over some set of assignments $Z$ (often $\zon$). The set of input distributions is some subset of all distributions over $X$ and the goal is to find an assignment  $\sigma\in Z$ that (approximately) maximizes the expected number of constraints: $\pr_{v\sim D}[v(\sigma) = 1]$. A more common formulation is to maximize the number of satisfied constraints drawn randomly from the input distribution. This is essentially equivalent since if the number of constraints $m =\Omega(\log(|Z|)/\eps^2)$ then for all assignments, the average number of random constraints satisfied by the assignment will be within $\eps$ of the expectation with high probability.

One example of such problems are planted random CSPs. In this case for every assignment $\sigma\in Z$ a distribution $D_\sigma$ is defined which depends on $\sigma$ and uniquely identifies $\sigma$ (for example the uniform distribution over predicates that $\sigma$ satisfies). Now, given access to input distribution $D_\sigma$, the goal is to recover $\sigma$. A potentially easier goal is to distinguish all planted distributions from some fixed distribution (most commonly uniform over all predicates). A related harder problem is to distinguish all distributions over predicates whose support can be satisfied by some assignment from some fixed (say uniform) distribution. In the context of $k$-SAT this problem is referred to as refutation.
See \cite{FeldmanPV:13} for an overview of the literature and a more detailed discussion.
\paragraph{Stochastic optimization:}
Supervised learning and random constraint satisfaction problems are special cases of stochastic optimization problems. Here the domain $X$ is that of some real-valued cost functions over the set of solutions $\F$. The set of input distributions is some subset of $S^X$ and the goal is to find a solution that approximately (for some notion of approximation) minimizes the expected cost, or $\E_{v\sim D}[v(f)]$.
One important class of such problems is stochastic convex optimization. Here $\F$ is some convex set in $\R^d$ and $X$ contains some subset of convex functions on $\F$ (such functions with range is $[-1,1]$). The set of input distributions usually contains all distributions over $X$. A detailed treatment of this type of problems can be found in \cite{FeldmanGV:15}.
\paragraph{Planted $k$-bi-clique:}
Let $k$ and $n$ be integers. For some (unknown) subset $S \subset [n]$ of size $k$ we are given samples from distribution $D_S$ over $\zon$ defined as follows: Pick a random and uniform vector $x \in \zon$; with probability $1-k/n$ output $x$ and with probability $k/n$ for all $i\in S$ set $x_i=1$ and then output $x$. Samples from this distribution can be seen as the rows of an adjacency matrix of a bipartite graph in which approximately $k/n$ fraction of vertices on one side are connected to all vertices in some subset $S$ of size $k$ and the rest of edges are random and uniform. The goal in this problem is to discover the set $S$ given access to distribution $D_S$. A potentially simpler problem is to distinguish all distributions in the set $\D =\{D_S \cond S\subseteq [n], |S|=k\}$ from the uniform distribution over $\zon$.

It is not hard to see that all the problems above are either decision problems or linear optimizing search problems. In addition, PAC learning, many settings of random constraint satisfaction and planted bi-clique are verifiable search problems. To see this in the case of planted bi-clique the query for set $S$ checks that all values in the set are set to 1 and the threshold is $k/n$. Some examples of the problem that is neither many-vs-one decision nor optimization is property testing for distributions and mean vector estimation studied in \cite{FeldmanGV:15}.

\section{Applications to other models}
\subsection{Memory-limited streaming}
\label{sec:memory}
In a streaming model with limited memory at step $i$ an algorithm observes sample $x_i$ drawn i.i.d.~from the input distribution $D$ and updates its state from $S_i$ to $S_{i+1}$, where for every $i$, $S_i \in \zo^b$. The solution output by the algorithm can only depend on its final state $S_n$.  \citet{SteinhardtVW16} showed that upper bounds on SQ complexity of solving a problem imply upper bounds on the amount of memory needed in the streaming setting. Specifically, they demonstrate that (their result is stated in a somewhat more narrow context of learning but can be easily seen to apply to general search problems):
\begin{thm}[\cite{SteinhardtVW16}]
\label{thm:low-memory-svw}
Let $\Z$ be a search problem over a finite set of distributions $\D$ on a domain $X$ and a set of solutions $\F$. Assume that $\Z$ can be solved using $q$ queries to $\STAT(\tau)$. Then for every $\delta >0$, there is an algorithms that solves $\Z$ with probability $\geq 1-\delta$  using
$O\lp\frac{q \cdot \log|\D|}{\tau^2} \cdot \log(q\log(|\D|)/\delta)\rp$ samples and
$O(\log|\D| \cdot \log(q/\tau))$ bits of memory.
\end{thm}

Our characterization of the deterministic search problems implies that the linear dependence of memory and sample complexity on $\log |\D|$ can be replaced with $\KLR(\D)/\tau^2$.
\begin{thm}
\label{thm:low-memory-det}
Let $\Z$ be a search problem over a finite set of distributions $\D$ on a domain $X$ and a set of solutions $\F$. If $\QC(\Z,\STAT(\tau)) \leq q$ then for every $\delta >0$, there is an algorithms that solves $\Z$ with probability $\geq 1-\delta$  using
$O\lp\frac{q \cdot \KLR(\D)}{\tau^4} \cdot \log(q/(\tau\delta))\rp$ samples and
$O\lp\frac{\KLR(\D)}{\tau^2} \cdot \log(q)\rp$ bits of memory.
\end{thm}
\begin{proof}
By Thm.~\ref{thm:stat-search-lower}, $\SD_\dci(\Z,\tau) \leq q$. We now demonstrate how to implement the algorithm in the proof of Thm.~\ref{thm:stat-search-upper} using samples and low memory.
At each step of the MW algorithm, given that we can compute $D_t$ we can also compute $f$ and the $d\ln(|\D|)$ queries that ``cover"  $\D\setminus \Z_f$. We can estimate the answer to each of these queries with tolerance $\tau/3$ and confidence $1-\delta'$ using $O(\log(1/\delta')/\tau^2)$ samples. Each estimation requires just $\log(\tau/3)$ bits of memory for a counter that can be reused.  We estimate the expectations until we find a query $\phi_i$ such that our estimate $v_i$ satisfies $\left|D_t[\phi_i] - v_i\right| >2\tau/3$ (and we do not need to remember estimates that do not satisfy the condition). If we find such $i$, we remember the index $i$ and the sign of $D_t[\phi_i] - v_i$. The index and the sign allow to reconstruct the function $\psi_t$ that is used for the MW update. Remembering them requires $\log(d\ln(|\D|)) + 1$ bits. If we do not find $i$ that satisfies this condition we output $f$ (that depends only on $D_t$).

This algorithm has at most $\frac{36\KLR(\D)}{\tau^2}$ steps. Thus the total memory required by the algorithm is $O(\frac{\KLR(\D)}{\tau^2} \cdot \log(q\log(|D|)))$. In the course of this algorithm we need to estimate the answers to $36\KLR(\D) \cdot d\ln(|\D|)/\tau^2$ queries. To ensure that all the estimates are correct with probability at least $1-\delta$ we need to choose $\delta' = \delta \tau^2/(36\KLR(\D) \cdot d\ln(|\D|))$. This implies that the total number of samples used by the algorithm is
$O\lp\frac{q \cdot \log(|D|) \cdot \KLR(\D)}{\tau^4} \cdot \log(q\log(|\D|)/(\tau\delta))\rp$.

These bounds are not as strong as what we claim due to an additional $\ln(|\D|)$ factor that we incurred in the characterization of decision problems via $\RSD$. However it can be easily eliminated by using the tight characterization of the SQ complexity of decision problems using the deterministic cover $\icvr(\D,D_0,\tau)$ that we described in Lemma \ref{lem:det-algorithm2queries}. Plugging deterministic cover into the characterization of deterministic SQ complexity of search problems we obtain a tighter characterization using $$\sup_{D_0 \in S^X} \inf_{f\in \F} \icvr(\D\setminus \Zf,D_0,\tau).$$ The resulting algorithm will produce a set of queries of size $q$ at every step leading to the stronger bounds that we claimed.
\end{proof}
The algorithm that we have obtained is not polynomial-time and it is interesting whether a polynomial-time reduction with similar properties exists. For most natural problems our bounds are at most polynomially worse than Thm.~\ref{thm:low-memory-svw}, whereas the improvement from $\log|\D|$ to $\KLR(\D)/\tau^2$ is exponential in many settings of interest.

Our characterization can also be used to obtain an analogue of Thm.~\ref{thm:low-memory-det} for randomized algorithms. In this case we will need to allow the streaming algorithm access to a random string that does not contribute to its memory use (note that in the results we stated the probability is solely over the randomness of the samples).

\paragraph{Sparse linear regression:}
The main application of Thm.~\ref{thm:low-memory-svw} in \cite{SteinhardtVW16} is for the problem of $k$-sparse least squares regression. In this problem we are given a set of i.i.d.~samples $(z_1,y_1),\ldots, (z_n,y_n) \in [-1,1]^d \times [-R,R]$ for some $R=O(k)$. The goal is to find $w$ such that $\|w\|_1\leq k$ and $w$ $\eps$-approximately minimizes
$L(w) \doteq \E_{(z,y)\sim D}[(wz-y)^2]$, namely, $L(w) \leq \min_{\|w'\|_1\leq k} L(w') + \eps$.
Using a SQ algorithm for sparse linear regression \citet{SteinhardtVW16}, demonstrate a streaming algorithm for a restricted setting of sparse linear regression whose memory requirement depends only logarithmically on the dimension (and polynomially on $k$, $1/\eps$. Specifically, they assume that the marginal distribution over $[-1,1]^d$ be fixed and the labels are equal to $zw^* + \eta$ for some $k$-sparse $w^*$ and fixed zero-mean random variable $\eta$. This allows them to ensure that, after appropriate discretization, $\log |\D|=O(k\log d)$.

We first note that our result allows removing all restriction on the label. We can discretize the values of the label to multiples of $\eps/2$, resulting in a domain of size $O(k/\eps)$. The space of all distributions over a domain of this size has KL-radius of $\log(k/\eps)$ and will not affect complexity in a significant way. We cannot similarly remove the assumption on the marginal distribution over the points in $[-1,1]^d$ since its KL-radius is linear in $d$. However we can allow fairly rich set of distributions such as a low-dimensional subspace or additional $\ell_1$-norm constraint. Specifically, if the marginal of each distribution in $\D$ is supported over vectors $z$ such that $\|z\|_1 \leq r$ then we can discretize the domain $[-1,1]^d$ to be of size $(dk/\eps)^{O(r/\eps)}$ (we only need each coordinate up to $\eps/2$, hence there are at most $2r/\eps$ non-zero coordinates out of $d$). This leads to the following theorem that generalizes the results from \cite{SteinhardtVW16}:
\begin{thm}
Let $\Z(k,r,\eps)$ be the problem of $k$-sparse least squares regression with error $\eps$ in which the input distribution is supported on pairs $(z,y)\in [-1,1]^d\times [-R,R]$ such that  $\|z\|_1\leq r$ and $R = O(k)$. There exists an algorithm that for every $\delta > 0$, solves $\Z(k,r,\eps)$ given
$\tilde O(d  k^4 r \log(1/\delta)/\eps^5)$ samples and $\tilde O(\log d \cdot  k^2 r /\eps^3)$ bits of memory.
\end{thm}
 Some of the dependencies on $k$ and $\eps$ in this result are worse than those obtained in \cite{SteinhardtVW16}.  In addition, \citet{SteinhardtVW16} demonstrate a technique for improving the number of samples from being linear in $d$ to polynomial in $r$.

\subsection{Limited communication from samples}
\label{sec:comm}
For an integer $b>0$, a $b$-bit sampling oracle $\COMM_D(b)$ is defined as follows: Given any function $\phi: X \rar \zo^b$, $\COMM_D(b)$  returns $\phi(x)$ for $x$ drawn randomly and independently from $D$, where $D$ is the unknown input distribution. This oracle was first studied by \citet{Ben-DavidD98} as a {\em weak Restricted Focus of Attention} model. They showed that algorithms in this model can be simulated efficiently using statistical queries and vice versa. Lower bounds against algorithms that use such an oracle have been studied in \cite{FeldmanGRVX:12,FeldmanPV:13}. \citet{FeldmanGRVX:12} give a tighter simulation of $\COMM(1)$ oracle using the $\VSTAT$ oracle instead of $\STAT$. This simulation was extended to $\COMM_D(b)$ in \cite{FeldmanPV:13} at the expense of factor $2^b$ blow-up in the SQ complexity. More recently, motivated by communication constraints in distributed systems, the sample complexity of several basic problems in statistical estimation has been studied in this and related models \cite{ZhangDJW13,SteinhardtD15,SteinhardtVW16}.

We start the by stating the simulation results formally:
\begin{thm}[\cite{FeldmanPV:13}]
\label{thm:unbiased-from-vstat}
Let $\Z$ be a search problem, $b,\beta > 0$ and $n = \RQC(\Z,\COMM(b),\beta)$. Then for any $\delta \in (0,1/4]$,
$\RQC(\Z,\VSTAT(n\cdot 2^b/\delta^2),\beta-\delta) =O(n \cdot 2^b)$.
\end{thm}

\begin{thm}[\cite{FeldmanGRVX:12}]
\label{thm:stat-from-unbiased}
Let $\Z$ be a search problem, $m,\beta > 0$ and $q = \RQC(\Z,\VSTAT(m),\beta)$. Then for any $\delta >0$,
$\RQC(\Z,\COMM(1),\beta-\delta) = O(qm\cdot \log(q/\delta))$.
\end{thm}

We characterize the query complexity of solving problems with $b$-bit sampling oracle using the combined statistical dimension with $\dcvi$-discrimination that we defined in Sec.~\ref{sec:combined} (see Remark \ref{rem:combined-for-vstat}).
\begin{defn}\label{def:csdim-vstat}
   For a decision problem $\B(\D,D_0)$, the \textbf{combined  statistical dimension} with $\dcvi$-discrimination of $\B(\D,D_0)$ is defined as $$\CRSD_\dcvi(\B(\D,D_0)) \doteq \sup_{\mu \in S^\D} \lp \dcvi(\mu,D_0)\rp^{-1}.$$
    For a search problem $\Z$ and $\alpha>0$, it is defined as  $$\CRSD_\dcvi(\Z,\alpha) \doteq \sup_{D_0 \in S^X}\inf_{\cP \in S^\F} \CRSD_\dcvi(\B(\D\setminus \Z_\cP(\alpha),D_0)) .$$
\end{defn}

We get the following corollaries by combining our lower bounds with the simulation results above:
\begin{cor}
\label{cor:1stat-decision}
Let $\B(\D,D_0)$ be a decision problem, $\tau > 0, \delta \in (0,1/2), b>0$  and let $d = \CRSD_\dcvi(\B(\D,D_0))$.
Then $$\RQC(\B(\D,D_0),\COMM(b),2/3) =\Omega(d^{2/3}/2^b)\mbox{ and}$$
$$\RQC(\B(\D,D_0),\COMM(1),1-\delta) = \tilde{O}(d^2 \cdot \ln^2(1/\delta)).$$
\end{cor}
\begin{proof}
To obtain the first part we apply the first part of Thm.~\ref{thm:combined-decision} with $\tau = d^{-1/3}$ and $\delta=1/4$ to get
$$\RQC(\B(\D,D_0),\vSTAT(d^{-1/3}),3/4) \geq d^{2/3}/2 .$$
By Lemma \ref{lem:vstat-reduction}, we then obtain that
$$\RQC(\B(\D,D_0),\VSTAT(d^{2/3}/9),3/4) \geq d^{2/3}/2 .$$
Now, applying Thm.~\ref{thm:unbiased-from-vstat} with $\beta =2/3$ and $\delta = 1/12$, we obtain that there exists a constant $c >0$, such that if $\RQC(\B(\D,D_0),\COMM(b),2/3) < c \cdot d^{2/3}/2^b$ then
$$\RQC(\B(\D,D_0),\VSTAT(d^{2/3}/9),3/4) < d^{2/3}/2 ,$$ violating our assumption.

To obtain the second part we apply the second part of Thm.~\ref{thm:combined-decision} with confidence parameter $\delta/2$ to get:
$$\RQC(\B(\D,D_0),\vSTAT(1/(2d),1-\delta/2) \leq 2d \ln(2/\delta).$$
Now, using Thm.~\ref{thm:stat-from-unbiased} with confidence parameter $\delta/2$, we get that
$$\RQC(\B(\D,D_0),\COMM(1),1-\delta) = \tilde{O}(d^2 \ln^2(1/\delta)).$$
\end{proof}

We note that the gap between the upper and lower bounds is cubic with an additional factor $2^b$. Our upper bound also uses only the $1$-bit sampling oracle. A natural question for further research would be to find a characterization that eliminates these gaps. For general search problem we can analogously obtain the following characterization:
\begin{cor}
\label{cor:1stat-search}
  Let $\Z$ be a search problem, $\beta> \alpha>0, \tau >0$ and let $d = \CRSD_\dcvi(\Z,\alpha)$.
Then $$\RQC(\Z,\COMM(b),\beta) =\Omega(d^{2/3}(\beta -\alpha)/2^b)\mbox{ and}$$
$$\RQC(\Z,\COMM(1),\alpha-\delta) = \tilde{O}(d^5 \cdot \KLR(\D) \cdot \ln^2(1/\delta)).$$
\end{cor}

\remove{
Using our randomized statistical dimension for $\VSTAT$ we define the following parameter:
\begin{def}
Let $\B(\D,D_0)$ be a decision problem. We define $\CRSD_\dcv(\B(\D,D_0))$ to be the largest integer $d$ such that $\RSD_\dcv(\Z,1/d,2/3) \geq d$.
Let $\Z$ be a search problem. We define $\CRSD_\dcv(\Z)$ to be the largest integer $d$ such that $\RSD_\dcv(\Z,1/d,2/3) \geq d$.
\end{def}

now show how to characterize the complexity
 This model is known to be equivalent to the SQ model up to polynomials \cite{Ben-DavidD98,FeldmanGRVX:12,FeldmanPV:13,SteinhardtVW16} and therefore our characterization immediately implies a characterization for this model.

}

\section{Additional relationships}
\label{app:extras}
\remove{
\begin{lem}
\label{lem:SD-r}
Let $\Z$ be a search problem over a set of solutions $\F$ and a class of distributions $\D$ over a domain $X$ and let $\tau>0$. Then
  $\SD_\dci(\Z,\tau) \geq \RSD_\dci(\Z,\tau) - 1$.
\end{lem}
\begin{proof}
We first observe that if $\lp\max\left\{\mu(\F), \dci\frc(\mu,D_0,\tau)\right\}\rp^{-1} = d$ then for every $f \in \F$, $\mu(\Z_f) \leq 1/d$. This means that for every  $\D'$, $ \mu(\D' \cond \D\sm\Z_f) \leq \mu(\D')/(1-1/d)$. This implies that  $$\dci\frc(\mu_{|\D\sm\Zf},D_0,\tau) \leq  \frac{ \dci\frc(\mu,D_0,\tau) }{1 - \fr{d}}  \leq \fr{d-1}.$$
Thus,
$$\sup_{D_0 \in S^X, \ \mu \in S^\D} \min_{f\in \F} \lp \dci\frc(\mu_{|\D\setminus \Zf},D_0,\tau)\rp^{-1} \geq \sup_{D_0 \in S^X, \ \mu \in S^\D} \lp\max\left\{\mu(\F), \dci\frc(\mu,D_0,\tau)\right\}\rp^{-1} -1.$$
The expression on the right is precisely  $\RSD_\dci(\Z,\tau) -1$. Now using the max-min inequality,
$$\sup_{D_0 \in S^X} \min_{f\in \F} \sup_{\mu \in S^\D}  \lp\dci\frc(\mu_{|\D\setminus \Zf},D_0,\tau) \rp^{-1}\geq \sup_{D_0 \in S^X, \ \mu \in S^\D} \min_{f\in \F} \lp \dci\frc(\mu_{|\D\setminus \Zf},D_0,\tau)\rp^{-1}  \geq \RSD_\dci(\Z,\tau) -1 .$$
The expression on the left is equal to
$$\sup_{D_0 \in S^X} \min_{f\in \F} \sup_{\mu \in S^{\D\sm\Zf}} \lp \dci\frc(\mu,D_0,\tau)\rp^{-1} = \sup_{D_0 \in S^X} \min_{f\in \F} \RSD_\dci(\B(\D\sm\Zf,D_0),\tau) =  \SD_\dci(\Z,\tau).$$
\end{proof}
}

\begin{lem}
\label{lem:stat-rand-to-det-cover}
If $\ircvr(\D,D_0,\tau)\leq d$ then for every measure $\mu$ over $\D$ and $\delta> 0$, there exists $\D_\delta\subseteq \D$ such that $\mu(\D_\delta) \geq 1-\delta$ and $$\icvr(\D_\delta,D_0,\tau) \leq d\ln(1/\delta).$$ In particular, $\icvr(\D,D_0,\tau) \leq d\ln(|\D|)$.
\end{lem}
\begin{proof}
Let $\cQ$ be the probability measure over functions such that for every $D\in \D$,
$$\pr_{\phi \sim \cQ}\lb \dif > \tau\rb \geq \frac{1}{d} .$$
For $s = d\ln(1/\delta)$ and every $D \in \D$, \equ{\pr_{\phi_1,\ldots,\phi_s \sim \cQ^s}\lb\exists i\in[s],\ \difp{\phi_i} > \tau \rb \geq 1 - \lp 1 -\fr{d}\rp^s \geq 1 - e^{s/d} = 1-\delta. \label{eq:cover-high-prob}}
Therefore
$$\E_{D\sim \mu,\ \phi_1,\ldots,\phi_s \sim \cQ^s}\lb\exists i\in[s],\ \difp{\phi_i} > \tau \rb \geq 1-\delta.$$
This means that there exists a set of functions $\phi_1,\ldots,\phi_s$ such that for $$\D_\delta \doteq \{D\in \D \cond \exists i\in[s],\ \difp{\phi_i} > \tau\}$$ we have that $\mu(\D_\delta) \geq 1-\delta$. By definition, $\icvr(\D_\delta) \leq s$.

To get the second claim we apply the result to the uniform measure over $\D$.
\end{proof}

\begin{lem}
\label{lem:stat-search-lower-random-simple}
  Let $X$ be a domain and $\Z$ be a search problem over a set of solutions $\F$
  and a class of distributions $\D$ over $X$. For $\tau > 0$, reference distribution $D_0$ and a measure $\mu$ over $\D$
  let $d = \left( \dci\frc(\mu,D_0,\tau)\right)^{-1}$. Then any (randomized) algorithm that solves $\Z$ with probability at least $\beta$ over the choice of $D \sim \mu$ (and its randomness) using queries to $\STAT(\tau)$ requires at least $(\beta-\mu(\F))d$ queries, where $\mu(\F) \doteq \max_{f\in \F} \left\{\mu(\Zf)\right\}$.
\end{lem}
\begin{proof}
     Let $\A$ be a {\em deterministic} algorithm that uses $q$ queries to $\STAT(\tau)$ and solves $\Z$ with probability at least $\beta$ over the choice of $D \sim \mu$ (we can assume that $\A$ is deterministic since in the average-case setting the success probability of a randomized algorithm $\A$ is the expected success probability of $\A$ with its coin flips fixed).  Let $\D^+$ be the set of distributions on which $\A$ is successful, namely all distributions $D\in \D$ such that for all legal answers of $\STAT_D(\tau)$, $\A$ outputs a valid solution $f\in\Z(D)$.  By the properties of $\A$, we know that $\mu(\D^+) \geq \beta$.

     We simulate $\A$ by answering any query $\phi:X \rightarrow [-1,1]$ of $\A$ with value $D_0[\phi]$. Let $\phi_1,\ldots,\phi_q$ be the queries generated by $\A$ in this simulation and let $f$ be the output of $\A$.
    For every $D\in \D^+$ for which $f$ is not a valid solution, the answers based on $D_0$ cannot be valid answers of $\STAT_D(\tau)$. In other words, for every $D\in \D^+\setminus \Zf$, there exists $i\in[q]$ such that $|D[\phi_i] - D_0[\phi_i]| > \tau$. Let $$\D_i \doteq \left\{D \in \D^+\setminus \Zf \ \left|\  \left|D[\phi_i] - D_0[\phi_i]\right| > \tau\right.\right\}.$$ Then by definition of $\dci\frc$, $\mu(\D_i) \leq 1/d$.
    On the other hand, $\D^+ \subseteq \Zf \cup \bigcup_{i\in[q]} \D_i$, and hence $$\mu(\Zf) +\sum_{i\in[q]} \mu(\D_i) \geq \mu(\D^+) \geq \beta.$$ This implies that $q/d \geq \beta-\mu(\Zf)$ or, $q \geq d(\beta-\mu(\F))$.
\end{proof}

\begin{lem}
\label{lem:k1k2}
$$\dcvi(\mu,D_0)\geq \frac{1}{4} \cdot \dcii(\mu,D_0) \geq \frac{1}{2}\cdot\dcvi(\mu,D_0)^2. $$
\end{lem}
\begin{proof}
\alequn{
\dcvi(\mu,D_0) &= \max_{\phi :X \rar [0,1]} \E_{D\sim \mu} \lb \left| \sqrt{ D[\phi]}-\sqrt{D_0[\phi]}\right|\rb  \nonumber \\
&= \max_{\phi :X \rar [0,1]} \E_{D\sim \mu} \lb \frac{\dif }{\sqrt{ D[\phi]}+\sqrt{D_0[\phi]}} \rb   \nonumber \\
&\geq  \frac{1}{2} \cdot \max_{\phi :X \rar [0,1]} \E_{D\sim \mu} \lb \dif  \rb \label{eq:bound-kappa_v}\\
 &= \frac{1}{4} \cdot \max_{\phi :X \rar [-1,1]}  \E_{D\sim \mu} \lb \dif \rb \equiv \\
 \frac{1}{4} \cdot \dcii(\D,D_0) &= \frac{1}{2}  \cdot \max_{\phi :X \rar [0,1]}  \E_{D\sim \mu} \lb \left|\sqrt{ D[\phi]}-\sqrt{D_0[\phi]}\right| \cdot \left( \sqrt{ D[\phi]}+\sqrt{D_0[\phi]}\right) \rb  \\
 &\geq \frac{1}{2}\cdot \max_{\phi :X \rar [0,1]}  \E_{D\sim \mu} \lb \left( \sqrt{ D[\phi]}-\sqrt{D_0[\phi]}\right)^2 \rb \\
 &\geq \frac{1}{2}\cdot \left(\max_{\phi :X \rar [0,1]}  \E_{D\sim \mu} \lb \left| \sqrt{ D[\phi]}-\sqrt{D_0[\phi]}\right| \rb \right)^2\\
}
\end{proof}

\begin{lem}
\label{lem:k2-rho}
$\rho(\D,D_0) \geq (\dc(\D,D_0))^2$ and therefore $\SD_\dc(\B(\D,D_0),\tau) \geq \SD_\rho(\B(\D,D_0),\tau^2)$.
\end{lem}
\begin{proof}
Let $\tau \doteq \dc(\D,D_0)$ and $\phi$ be the function such that  $\|\phi\|_{D_0}  = 1$ and
\equn{
\frac{1}{|\D|} \cdot \sum_{D \in \D} \dif = \tau.
}
Then
\alequn{\tau^2  &= \frac{1}{|\D|^2} \cdot \lp \sum_{D \in \D} \dif \rp^2 \\
& = \frac{1}{|\D|^2} \cdot \lp \sum_{D \in \D} D_0[\hat{D} \cdot \phi] \cdot\sgn(D_0[\hat{D} \cdot \phi]) \rp^2 \\
& = \frac{1}{|\D|^2} \cdot \lp  D_0\lb \phi \cdot \sum_{D \in \D}\sgn(D_0[\hat{D} \cdot \phi]) \cdot \hat{D} \rb \rp^2 \\
& \leq \frac{1}{|\D|^2} \cdot  \|\phi\|_{D_0}^2 \cdot \left\|\sum_{D \in \D}\sgn(D_0[\hat{D} \cdot \phi]) \cdot \hat{D}\right\|_{D_0}^2  \\
& = \frac{1}{|\D|^2} \cdot  D_0\lb \lp \sum_{D \in \D}\sgn(D_0[\hat{D} \cdot \phi])\cdot \hat{D}\rp^2 \rb \\
& = \frac{1}{|\D|^2} \cdot  \sum_{D,D' \in \D} \sgn(D_0[\hat{D} \cdot \phi]) \cdot \sgn(D_0[\hat{D'} \cdot \phi]) \cdot D_0\lb \hat{D} \cdot \hat{D'}\rb \\
& \leq  \frac{1}{|\D|^2} \cdot  \sum_{D,D' \in \D} \left|D_0\lb \hat{D} \cdot \hat{D'}\rb\right| = \rho(\D,D_0).
}
\end{proof}

\end{document}